\documentclass{article}

\usepackage[main, final]{neurips_2026}

\usepackage[utf8]{inputenc} 
\usepackage[T1]{fontenc}    
\usepackage{amsmath}        
\usepackage{amsthm}         
\usepackage{hyperref}       
\usepackage{url}            
\usepackage{graphicx}       
\usepackage{booktabs}       
\usepackage{amsfonts}       
\usepackage{nicefrac}       
\usepackage{microtype}      
\usepackage{xcolor}         
\usepackage{helvet}         
\usepackage{tikz}           
\usetikzlibrary{arrows.meta,calc,fit,positioning}

\newcommand{\abs}[1]{\left|#1\right|}

\definecolor{figInk}{HTML}{272727}
\definecolor{figMuted}{HTML}{6B7280}
\definecolor{figLine}{HTML}{CFCECE}
\definecolor{figBlue}{HTML}{0F4D92}
\definecolor{figGreen}{HTML}{2E7D5B}
\definecolor{figGreenFill}{HTML}{DDF3DE}
\definecolor{figOrange}{HTML}{B56A2A}
\definecolor{figOrangeFill}{HTML}{F5E8D9}
\definecolor{figFloor}{HTML}{E5E7EB}

\newtheorem{theorem}{Theorem}
\newtheorem{proposition}{Proposition}[theorem]
\newtheorem{corollary}{Corollary}[theorem]

\title{Why Cross-Skeleton Retargeting Is Non-Identifiable: Structural Limits of Generative Motion Models}

\author{%
  Zhiyuan Li$^{1}$\quad Wenyan Yang$^{1}$\quad Pekka Marttinen$^{2}$\quad Joni Pajarinen$^{1}$\\
  $^{1}$Department of Electrical Engineering and Automation, Aalto University, Finland\\
  $^{2}$Department of Computer Science, Aalto University, Finland\\
  \texttt{\{zhiyuan.li, wenyan.yang, pekka.marttinen, joni.pajarinen\}@aalto.fi}
}

\begin{document}

\maketitle

\begin{abstract}
  Cross-skeleton motion generation trains generative models to carry action structure and motion intention from one body to another. Yet a target motion that shows the right action has two explanations that the training data cannot tell apart: the model transferred the source clip, or it recovered a typical motion for the requested action. We show that this ambiguity is structural rather than incidental: under standard generative objectives, the source-conditioned retargeting map is non-identifiable in sparse heterogeneous motion domains. Unpaired distribution matching yields \emph{gauge non-identifiability}: the latent spaces of different skeletons can be transformed relative to one another without changing the training evidence, so different source-conditioned maps fit it equally well. Sparse paired supervision admits the complementary failure mode, \emph{conditional-mean degeneration}: when clips are paired only by action, squared-error training converges to an average target motion that ignores the source clip. To make the missing evidence observable, we introduce \emph{Source-Instance Fidelity} (SIF), a diagnostic that tests whether outputs differ from one another the way their source clips do, with the target skeleton and action held fixed. Under this diagnostic, methods that succeed at the standard action-level test on animal motion data often sit at the \emph{source-blind floor}, while the methods that rise above it retain only a partial relational signal. Retargeting therefore needs objectives and evaluations that can identify the source-conditioned map it claims to learn. Project page: \url{https://cross-skeleton-retargeting.netlify.app/}.
\end{abstract}

\section{Introduction}
\label{sec:intro}

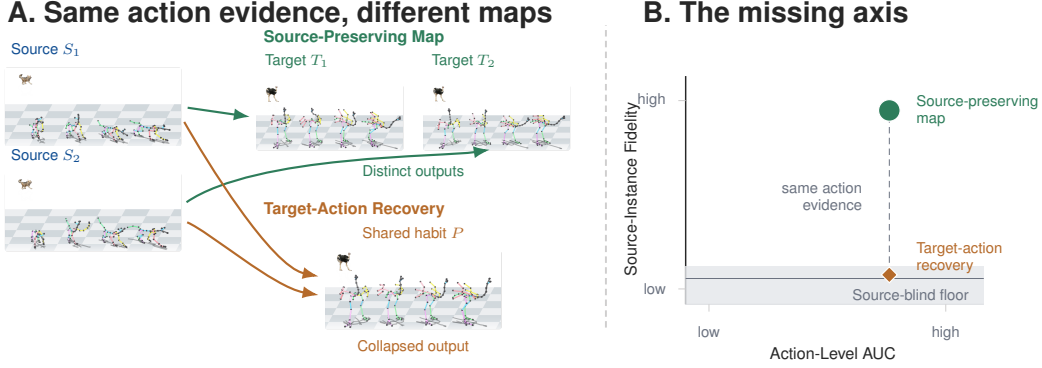
\begin{figure}[t]
  \centering
  \resizebox{0.99\linewidth}{!}{%
  \begin{tikzpicture}[
      x=1cm,
      y=1cm,
      >=Latex,
      figfont/.style={font=\fontfamily{phv}\selectfont},
      every node/.style={figfont},
      title/.style={figfont,font=\fontfamily{phv}\selectfont\bfseries\fontsize{13.2}{14.2}\selectfont,text=figInk},
      small/.style={figfont,font=\fontfamily{phv}\selectfont\fontsize{6.9}{7.8}\selectfont,text=figMuted},
      axislabel/.style={figfont,font=\fontfamily{phv}\selectfont\fontsize{7.5}{8.5}\selectfont,text=figInk},
      branch/.style={figfont,font=\fontfamily{phv}\selectfont\bfseries\fontsize{7.5}{8.5}\selectfont},
      card/.style={inner sep=0pt,outer sep=0pt,draw=figLine,line width=0.35pt,rounded corners=0.8pt,fill=white},
      greenarrow/.style={->,figGreen,line width=1.0pt,shorten >=3.2pt,shorten <=3.2pt},
      orangearrow/.style={->,figOrange,line width=1.0pt,shorten >=3.2pt,shorten <=3.2pt},
    ]

    \node[title,anchor=west] at (0.00,5.72) {A. Same action evidence, different maps};

    \node[small,text=figBlue,anchor=south west] at (0.06,4.92) {Source $S_1$};
    \node[card] (sone) at (1.52,4.22)
      {\includegraphics[width=2.86cm]{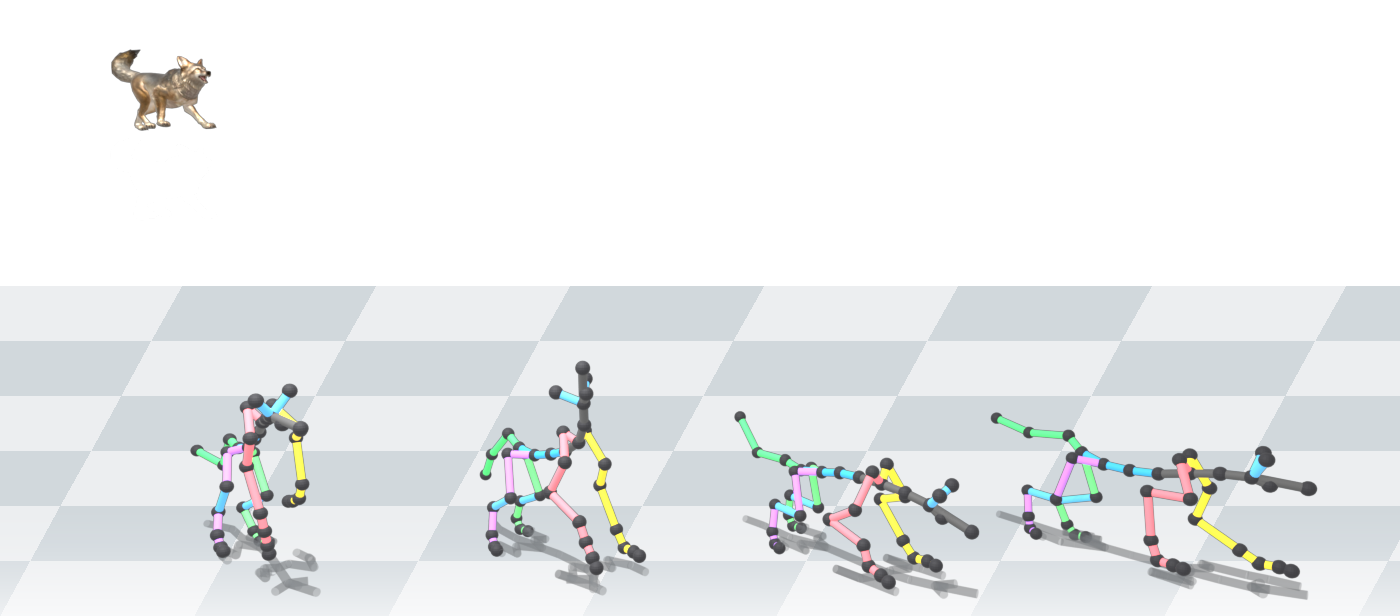}};
    \node[small,text=figBlue,anchor=south west] at (0.06,3.20) {Source $S_2$};
    \node[card] (stwo) at (1.52,2.50)
      {\includegraphics[width=2.86cm]{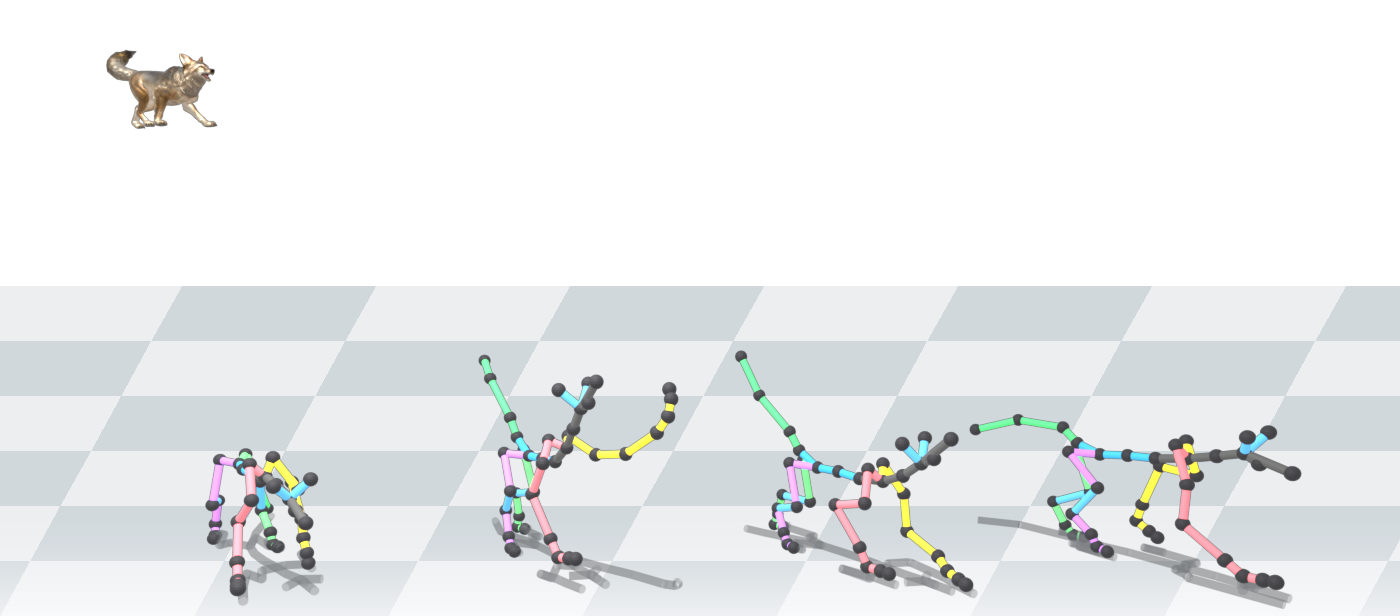}};

    \node[branch,text=figGreen,anchor=west] at (4.18,5.34) {Source-Preserving Map};
    \node[small,text=figGreen,anchor=south west] at (4.20,4.72) {Target $T_1$};
    \node[card] (tone) at (5.38,4.04)
      {\includegraphics[width=2.38cm]{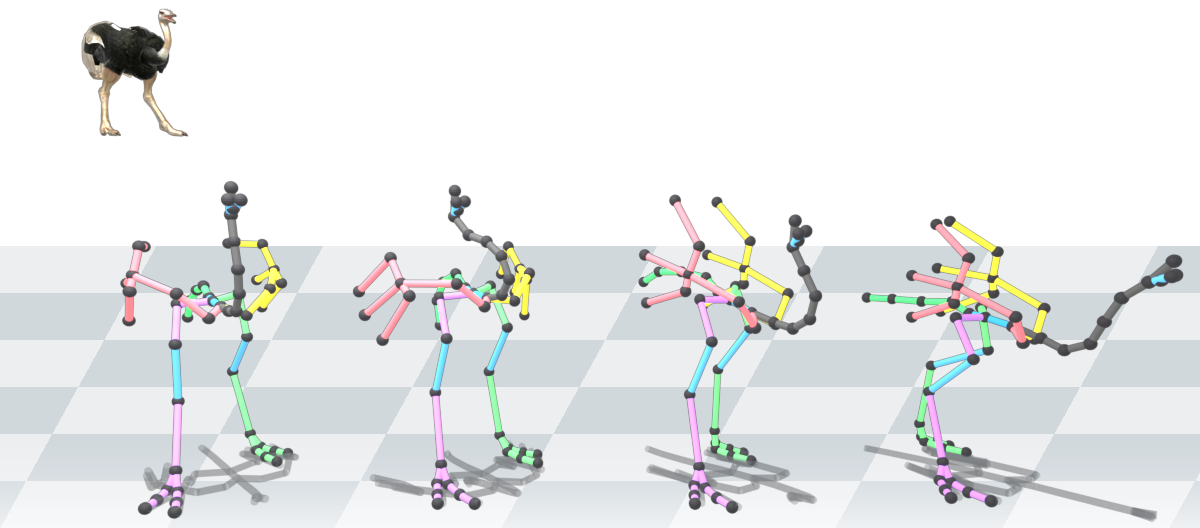}};
    \node[small,text=figGreen,anchor=south west] at (6.92,4.72) {Target $T_2$};
    \node[card] (ttwo) at (8.10,4.04)
      {\includegraphics[width=2.38cm]{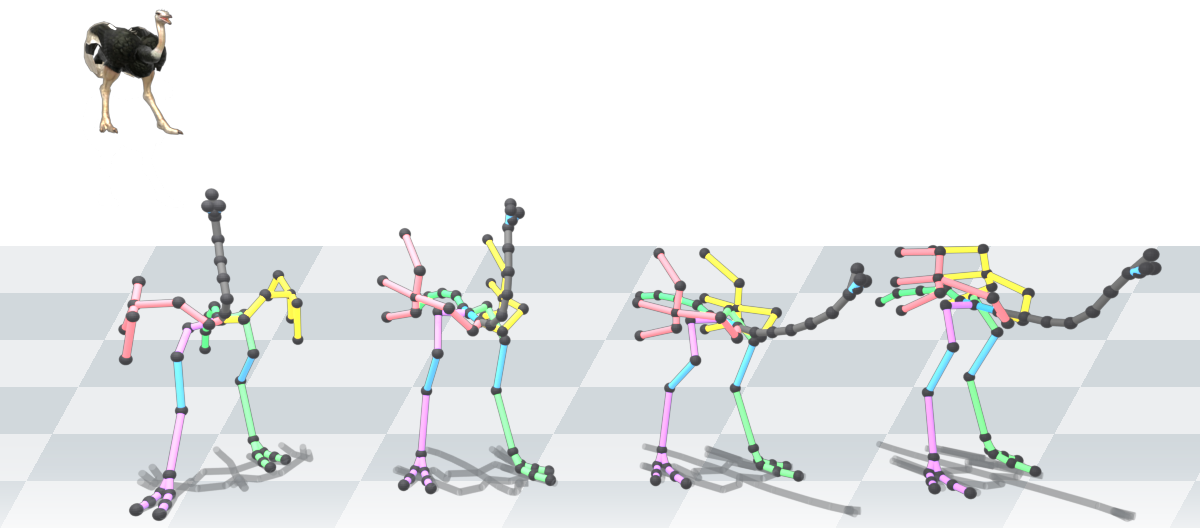}};
    \node[small,text=figGreen,anchor=center] at (6.74,3.16) {Distinct outputs};

    \node[branch,text=figOrange,anchor=west] at (4.18,2.54) {Target-Action Recovery};
    \node[small,text=figOrange,anchor=south] at (6.74,1.95) {Shared habit $P$};
    \node[card] (proto) at (6.74,1.23)
      {\includegraphics[width=2.86cm]{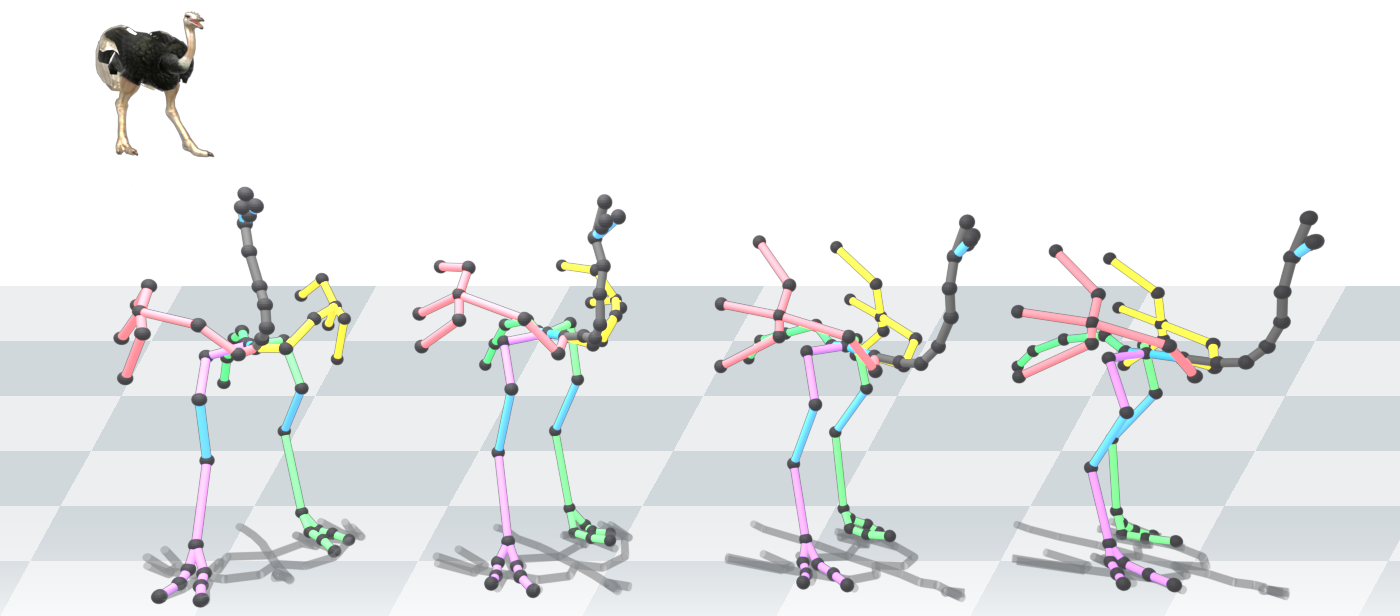}};
    \node[small,text=figOrange,anchor=center] at (6.74,0.30) {Collapsed output};

    \draw[greenarrow] (sone.east) -- (tone.west);
    \draw[greenarrow] ($(stwo.east)+(0.00,0.10)$) .. controls (4.18,3.44) and (7.22,3.40) .. ($(ttwo.south)+(0.00,0.02)$);
    \draw[orangearrow] ($(sone.east)+(0.00,-0.15)$) .. controls (3.36,3.35) and (4.42,1.62) .. ($(proto.west)+(0.00,0.20)$);
    \draw[orangearrow] ($(stwo.east)+(0.00,-0.12)$) .. controls (3.42,2.18) and (4.46,1.28) .. ($(proto.west)+(0.00,-0.10)$);

    \draw[figLine,line width=0.90pt,densely dashed] (9.88,0.62) -- (9.88,5.55);

    \node[title,anchor=west] at (10.36,5.72) {B. The missing axis};
    \coordinate (axisO) at (11.18,1.00);
    \coordinate (axisX) at (16.04,1.00);
    \coordinate (axisY) at (11.18,4.72);
    \draw[figLine,line width=0.58pt] (axisO) -- (axisX);
    \draw[figInk,line width=0.68pt] (axisO) -- (axisY);
    \fill[figFloor,opacity=0.82] (11.18,1.00) rectangle (16.04,1.62);
    \draw[figMuted,line width=0.44pt] (11.18,1.42) -- (16.04,1.42);
    \node[small,anchor=east,text=figMuted] at (15.92,1.16) {Source-blind floor};

    \draw[figLine,line width=0.45pt] (11.56,1.00) -- (11.56,0.88);
    \draw[figLine,line width=0.45pt] (15.42,1.00) -- (15.42,0.88);
    \draw[figLine,line width=0.45pt] (11.18,1.25) -- (11.06,1.25);
    \draw[figLine,line width=0.45pt] (11.18,4.32) -- (11.06,4.32);
    \node[small,anchor=north,text=figMuted] at (11.56,0.77) {low};
    \node[small,anchor=north,text=figMuted] at (15.42,0.77) {high};
    \node[small,anchor=east,text=figMuted] at (10.98,1.25) {low};
    \node[small,anchor=east,text=figMuted] at (10.98,4.32) {high};
    \node[axislabel,anchor=north] at (13.58,0.42)
      {Action-Level AUC};
    \node[axislabel,rotate=90,anchor=south] at (10.56,2.88)
      {Source-Instance Fidelity};

    \draw[figMuted,line width=0.45pt,densely dashed] (14.50,1.50) -- (14.50,4.16);
    \node[small,text=figMuted,anchor=east,align=right] at (14.18,2.76)
      {same action\\evidence};
    \filldraw[fill=figOrange,draw=white,line width=0.55pt]
      (14.37,1.48) -- (14.50,1.61) -- (14.63,1.48) -- (14.50,1.35) -- cycle;
    \node[small,text=figOrange,anchor=west,align=left] at (14.82,1.76) {Target-action\\recovery};
    \fill[figGreen,draw=white,line width=0.55pt] (14.50,4.16) circle (5.0pt);
    \node[small,text=figGreen,anchor=west,align=left] at (14.82,4.16) {Source-preserving\\map};
  \end{tikzpicture}}
  \caption{Action-level success does not identify the source-conditioned map.
  Two source clips with the same action can lead either to distinct target
  motions under a source-preserving map or to a shared target-action habit under
  source-blind recovery. Standard action-level AUC accepts both explanations;
  Source-Instance Fidelity tests whether the source-instance relation survives
  the change of body.}
  \label{fig:teaser}
\end{figure}

Motion generation increasingly has to work across bodies, not only on one fixed
skeleton. Methods have moved from retargeting models that encode skeletal
structure to generative and correspondence-based models built for many different
bodies
\citep{Aberman2020SAN,same2024,SkeletonInContext2023,raab2025anytop,m2m2024}.
A motion clip may come from a human, a quadruped, or another animal, and the
target body may differ in proportions, joint layout, and range of motion.
Cross-skeleton retargeting asks for a motion on the target body that keeps both
the action and the way the source clip performs it: its timing, its style, and
the details that set it apart from other clips of the same action. We call this
the motion intention of the source clip.

Progress in this area is usually measured by a retrieval test introduced with
large multi-body motion libraries \citep{raab2025anytop,anytop2025truebones}.
The generated motion serves as a query against real clips of the target body:
clips with the requested action count as matches, and clips with other actions
as non-matches. The score is the area under the receiver-operating-characteristic
curve (AUC), the probability that a match ranks above a non-match. This is a
sound test of whether the requested action can be recognized on the target body.
It does not test whether the output came from the particular source clip that
was given.

Retargeting makes a stronger claim: each source clip is carried to its own
target motion. We call this mapping the source-conditioned map.
Figure~\ref{fig:teaser} shows why the retrieval score cannot check it. Two
different coyote attacks should become two ostrich attacks that differ in the
same way. A model that ignores the source and returns the same habitual ostrich
attack for both, a target-action habit, produces outputs that are just as
recognizable as attacks. The action-level score accepts both, and only a second
question separates them: whether changing the source clip changes the output
accordingly.

The ambiguity becomes structural in sparse libraries of many different bodies.
In the animal motion data we study, the Truebones zoo
\citep{anytop2025truebones}, the skeletons differ in structure and joint count,
most combinations of body and action have no clip at all, and many of the rest
have only one. Such data show which actions each body can perform, but they
contain no pairs that say how a given source clip should move on another body.
The training data therefore do not determine which source-conditioned map is
correct.

We show that this is a limit of identifiability: the training data are
consistent with many different maps, and standard objectives give no reason to
prefer one \citep{LocatelloEtAl2019DisentanglementImpossible,khemakhem2020ivae}.
In this sense the source-conditioned map is \emph{non-identifiable}. Two
mechanisms produce the limit. When each body's motions are matched only in
distribution, through a latent space shared by all bodies, the latent space of
one body can be rotated or otherwise transformed relative to another without
changing the training loss, so different source-to-target maps fit the data
equally well. We call this \emph{gauge non-identifiability}. When training pairs
each source clip with a randomly chosen target clip of the same action,
squared-error training converges to the average target motion for that action,
whatever the source. We call this \emph{conditional-mean degeneration}. Both
mechanisms explain why better action recovery does not by itself identify the
retargeting map.

To measure what the retrieval score misses, we introduce
\emph{Source-Instance Fidelity} (SIF). For a fixed source body, target body, and
action, SIF asks whether the outputs differ from one another the way their
source clips do: two source clips that are far apart should lead to outputs that
are far apart too. A model that ignores the source scores near zero, which we
call the \emph{source-blind floor}. SIF gives the true mapping a score near one
on synthetic data with known correspondences, and a similarly high score on real
human-to-robot pairs.

Measured this way, most methods fall short. On 1,891 combinations of source
body, target body, and action from Truebones, nine of fourteen evaluated methods
sit at or near the source-blind floor, including methods that score well on the
retrieval test. The same reading also qualifies the methods that rise above the floor. ACE~\citep{ace2023},
an adversarial cross-embodiment model that adds a motion-space constraint beyond
per-skeleton distribution matching, retains only a partial relational signal:
its outputs keep some of the ordering among source clips but vary little with
them.

Retargeting therefore cannot be judged only by whether the target action is
recognizable. A generated motion can look plausible and show the right action
while ignoring how the source clip moved. What has to be identified is the
source-conditioned map itself, and the central question is which evidence and
inductive biases are enough to select it.

\section{Related Work}
\label{sec:related}

\paragraph{Cross-skeleton motion retargeting.}
The retargeting literature has progressively weakened the assumption that
source and target bodies share a convenient representation. Early neural
retargeting made skeletal structure explicit, while later skeleton-agnostic
representations and context-conditioned sequence models sought motion spaces
that remain meaningful across body graphs
\citep{Aberman2020SAN,same2024,SkeletonInContext2023}. More recent work carries
the same ambition into sparse heterogeneous libraries, where generative models
and sparse-correspondence procedures must operate without dense paired examples
\citep{raab2025anytop,moreflow2025,ace2023,m2m2024}. Our concern is orthogonal
to this modeling trajectory: once an output has the requested action, the
remaining question is whether the evidence identifies the particular
source-conditioned map that retargeting claims.

\paragraph{Motion generation beyond fixed bodies.}
Recent motion generation work also shows that the fixed human skeleton is no
longer the only natural domain of synthesis. The shared pressure across this
broader literature is to represent motion when body identity, morphology, and
topology become part of the conditioning problem rather than fixed background
structure \citep{raab2024single,lee2025how,egan2024dogcode,chen2025topologyagnostic}.
These settings make body variation a central modeling object. The structural
question here is complementary: when data are sparse and correspondences are
not instance-level, action labels and target examples may specify what action
should appear on the target body without specifying which source motion
intention selected it.

\paragraph{Identifiability and distributional transport.}
Our formal argument is closest in spirit to identifiability limits in
representation learning, where marginal fit alone does not recover latent
structure without suitable inductive bias or auxiliary variables
\citep{LocatelloEtAl2019DisentanglementImpossible,khemakhem2020ivae}.
Distributional transport gives a second comparison point, since Schr\"odinger
bridge objectives can learn a coupling between marginals without observing
instance-level correspondences \citep{DeBortoliEtAl2021DSB,unsb2023}. In
cross-skeleton retargeting, however, the missing object is not merely a
transport between two marginal motion distributions. It is the relative
alignment that tells a target decoder how to interpret the source instance.

\paragraph{Motion libraries and retrieval comparators.}
Motion-library methods preserve realized target-side motion snippets by
construction, and this property connects classical motion graphs to modern
sparse-correspondence transfer \citep{KovarGleicher2002MotionGraphs,m2m2024}.
This makes retrieval useful as a diagnostic foil: it can expose what
action-level evidence accepts without claiming to synthesize a new
source-conditioned motion. We use this distinction later to define a
target-library comparator whose role is diagnostic rather than generative.
Appendix~\ref{app:related} gives additional context for these connections.

\section{The Source-Conditioned Map}
\label{sec:map}

Retargeting concerns a map that takes each source clip to its own target
motion, not only a distribution of plausible target motions. Let $\mathcal{X}_s$ denote the motion space of
skeleton $s$, with clips annotated by an action label $c$; this notation records
that clips on different bodies occupy different kinematic domains. A retargeting claim
asserts a family of maps $T_{a\to b}:\mathcal{X}_a\to\mathcal{X}_b$ such that a
source clip $x_a$ on skeleton $a$ is carried to skeleton $b$ while preserving
the action, timing, and motion intention of that clip. This is a
stronger object than a conditional generator for plausible motions on $b$ given
the label $c$: the map must change when the source instance changes, even when
the action and target skeleton are held fixed.

Sparse heterogeneous motion libraries provide much weaker evidence than this
map requires. In the Truebones zoo used in our evaluation
\citep{anytop2025truebones,raab2025anytop}, the dataset contains 70 skeletons and
616 clips across 90 exact actions, yielding 6,300 possible skeleton-action
cells. Only 413 of these cells are occupied, a fill rate of 6.6\%, and the
median occupied cell contains one clip. Thus a learner often sees that a
skeleton can perform an action, but not multiple within-action instances from
which the geometry of motion intention could be inferred, much less a
dense set of source-target correspondences specifying how the same instance
should appear on a different body.

This sparsity matters because source-instance evidence is relational. To ask
whether a map preserves the source instance, one must be able to compare several
source clips that share a source skeleton and action, generate the corresponding
motions on the same target skeleton, and then ask whether the source-side
geometry remains visible after retargeting. Truebones contains only 49 such
source groups, a source skeleton and action with at least three clips. Pairing
each group with every other skeleton that performs the same action yields 1,891
source-target-action triples; these add target bodies but no new source motions,
so all of them draw on the same 171 source clips, spread over eight actions in
which attack and idle account for 94 percent of the triples. The structural question is therefore not whether
action labels are useful, since they plainly are, but whether this sparse
evidence selects a source-conditioned map rather than a target-action habit.

The distinction fixes the burden for the rest of the paper. Action-label
evidence constrains which target motions are admissible for a requested action,
whereas source-instance evidence constrains how differences among same-action
source clips should reappear after the body changes. The former can support
recognizable cross-skeleton generation; only the latter can identify the map
that a retargeting claim asserts. Section~\ref{sec:structural} makes this gap
precise by showing how standard objectives can leave multiple source-conditioned
maps consistent with the same observed evidence.

\section{Structural Non-Identifiability}
\label{sec:structural}

A training objective can reward plausible target motions without caring which
source-conditioned map produced them. This section makes
the previous ambiguity structural by isolating two objective-level mechanisms:
per-skeleton marginal matching leaves a relative latent coordinate freedom
between bodies, while cell-paired squared-error regression under random
target-cell pairing selects a target-action prototype. Both mechanisms preserve
the evidence consumed by the objective while changing, or erasing, the
source-instance dependence that retargeting requires.

The first mechanism appears whenever skeleton-specific encoders and decoders are
coupled only through per-skeleton latent marginals. Let $P_s$ be the motion
distribution for skeleton $s$, let $\mu_s=(E_s)_\#P_s$ be the latent marginal,
and write a shared-latent objective schematically as
\begin{equation}
  L(E,D) =
  \sum_s L_s(D_s\circ E_s;P_s) + \Omega(\mu_1,\ldots,\mu_K),
  \label{eq:shared_latent_obj}
\end{equation}
where $\Omega$ depends only on the per-skeleton latent marginals. For measurable
invertible latent transformations $g=(g_1,\ldots,g_K)$ with measurable
inverses, define
\begin{equation}
  G_\Omega(\mu) =
  \{g:\Omega((g_1)_\#\mu_1,\ldots,(g_K)_\#\mu_K)
  =\Omega(\mu_1,\ldots,\mu_K)\}.
  \label{eq:gauge_group}
\end{equation}
For $g\in G_\Omega(\mu)$, set $E'_s=g_s\circ E_s$ and
$D'_s=D_s\circ g_s^{-1}$. The relative transformations
$g_b^{-1}\circ g_a$ are invisible to the marginal objective, but they are
exactly the transformations that determine how a source latent is interpreted by
a target decoder.

\begin{theorem}[Gauge non-identifiability]\label{thm:gauge}
Let $L$ be the objective in Eq.~\eqref{eq:shared_latent_obj} with $\Omega$
depending only on per-skeleton latent marginals, and let
$g\in G_\Omega(\mu)$ as in Eq.~\eqref{eq:gauge_group}. Then
$L(E',D')=L(E,D)$. If, in addition, there exist $a,b$ such that
$D_b\circ g_b^{-1}\circ g_a\neq D_b$ on a set of positive $\mu_a$-measure, the
induced retargeting maps differ:
$T'_{a\to b}=D_b\circ g_b^{-1}\circ g_a\circ E_a\neq
T_{a\to b}=D_b\circ E_a$ on a set of positive $P_a$-measure. $T_{a\to b}$ is
determined by $L$ only up to the relative gauge $g_b^{-1}\circ g_a$.
\end{theorem}

This theorem is narrower than a generic impossibility result for disentanglement
\citep{LocatelloEtAl2019DisentanglementImpossible} or nonlinear independent
component analysis. It does not say that no auxiliary information can identify a
representation. It says that the objective class above, when supplied with
sparse skeleton-and-action evidence, does not select a relative latent alignment
between skeletons. Identifiability under iVAE-style auxiliary variables
\citep{khemakhem2020ivae} requires the auxiliary distribution to induce
full-rank variation in the prior sufficient statistics; on Truebones, where
only 413 skeleton-action cells out of 6,300 are occupied and the median occupied
cell contains one clip, skeleton-and-action metadata do not provide the
empirical support needed to invoke that condition.

For Gaussian-aligned regularizers, the invariance contains $O(d)^K$; with
$K=70$ skeletons and the $d=32$ latent dimension used for the finite-dimensional
lower-bound calculation, quotienting the diagonal $O(d)$ subgroup, one common
orthogonal factor applied to all skeletons, still leaves a lower bound of
34,224 continuous gauge degrees of freedom. This is not a count of distinct
realized maps, because an observable retargeting ambiguity requires
\emph{decoder non-degeneracy}: the decoder must discriminate at least some
directions on the gauge orbit. That condition is architecture-dependent. The
checks in Appendix~\ref{app:decoder_nondegeneracy} record a small
Kullback-Leibler (KL)-regularized VAE whose decoder nearly absorbs the orbit,
with a gauge-to-noise ratio of 1.11, and a separate AnyTop hidden-state sweep
showing that a transformer diffusion decoder for arbitrary skeleton topology
\citep{raab2025anytop} can make an orthogonal perturbation orbit visible after
decoding. The AnyTop sweep is not the source of the $d=32$ count; it is a
decoder-observability boundary check. Orbit observability is nevertheless
distinct from gauge-element selection: a decoder can distinguish rotation
directions without the objective selecting which rotation aligns source and
target latents. The main point is therefore not that every latent coordinate
change is observable, but that the objective does not identify the relative
latent alignment needed for a source-preserving map. In this sense
Theorem~\ref{thm:gauge} is a statement about identifiability: the objective
admits many source-conditioned maps consistent with the same evidence. It does
not predict that a trained model's outputs collapse, nor where a method falls on
SIF.

The second mechanism arises when sparse paired evidence is treated as if it were
instance-level supervision. Consider a predictor trained with squared error from
a source clip $x_a$ and target skeleton $b$ to a target motion $x_b$. If the
target is drawn uniformly from the target skeleton and source-action cell rather
than paired to the same motion instance, the risk is
\begin{equation}
  L_{\mathrm{paired}}(f)=
  \mathbb{E}\left[\|f(x_a,b)-x_b\|_2^2\right].
  \label{eq:l_paired}
\end{equation}
The Bayes-optimal squared-error predictor is then the conditional mean of the
target cell \citep{HastieTibshiraniFriedman2009ESL}; in this regime, the
conditional mean is a target-action prototype rather than a source-conditioned
motion.

\begin{proposition}[Conditional-mean degeneration]\label{prop:cm}
Assume $x_b$ is square-integrable, $f$ ranges over measurable functions, and the
empirical pairing draws targets uniformly at random from the (skeleton, action)
cell $B(b,c)=\{x : (\mathrm{skel}(x),\mathrm{action}(x))=(b,c)\}$ with
$c=\mathrm{action}(x_a)$. Then the (almost surely unique) minimiser of
$L_{\mathrm{paired}}$ is
$f^\star(x_a,b)=|B(b,c)|^{-1}\sum_{x_b'\in B(b,c)}x_b'$,
the empirical cell mean: a single fixed prototype per (target skeleton, source
action), independent of the specific source clip.
\end{proposition}

The proposition is deliberately scoped. It applies to cell-paired squared-error
regression under random target-cell pairing, not to true instance-level pairs,
which would supply exactly the correspondence the sparse regime lacks.
Section~\ref{sec:sif} tests this boundary on paired human-to-robot data: under
random target-cell pairing a model's outputs collapse toward the cell mean, while
the same model trained on true pairs recovers the source-instance relation. It
also assumes the target cell is nonempty and that motions are represented in the
Euclidean vector space on which the squared-error loss is evaluated. The
controlled ladder in Figure~\ref{fig:propcm_ladder} shows the same mechanism in
a setting where the correct source-conditioned transport $T^\star$ is known;
the oracle variance is the within-cell variance of $T^\star(x_a)$ over held-out
source clips. As paired support per cell ranges from $M=2$ to $M=50$, the
squared-error regressor trained with random same-cell targets keeps only about
2\% of that oracle within-cell output variance.
No evaluated Truebones method exactly instantiates this random cell-pairing
model. The closest real-data echo is AL-Flow, a label-conditional flow
comparator that receives cluster and exact-action labels but no source motion.
The latent-level diagnostic of Section~\ref{sec:sif} (Table~\ref{tab:lsif}) places AL-Flow at
$-0.098$, indicating that source-instance information can be lost before
decoding when the training signal approximates cell-level pairing rather than
instance-level correspondence. Appendix~\ref{app:data_methods} places each
evaluated method relative to Theorem~\ref{thm:gauge} and
Proposition~\ref{prop:cm}; for these methods, the collapse of output variation
reported in Section~\ref{sec:sif} is an empirical observation rather than a
prediction of either result.
Appendix~\ref{app:proofs} contains the formal proofs: the gauge theorem in
Appendix~\ref{app:gauge_proof}, the conditional-mean proposition in
Appendix~\ref{app:cm_proof}, and the retrieval corollary in
Appendix~\ref{app:retr_gauge_proof}.

\begin{figure}[t]
  \centering
  \resizebox{\textwidth}{!}{%
  \begin{tikzpicture}[
      x=1cm,
      y=1cm,
      >=Latex,
      figfont/.style={font=\fontfamily{phv}\selectfont},
      every node/.style={figfont},
      title/.style={figfont,font=\fontfamily{phv}\selectfont\bfseries\fontsize{12.8}{14.0}\selectfont,text=figInk},
      branch/.style={figfont,font=\fontfamily{phv}\selectfont\bfseries\fontsize{7.6}{8.6}\selectfont},
      small/.style={figfont,font=\fontfamily{phv}\selectfont\fontsize{6.7}{7.6}\selectfont,text=figMuted},
      axislabel/.style={figfont,font=\fontfamily{phv}\selectfont\fontsize{7.2}{8.2}\selectfont,text=figInk},
      card/.style={inner sep=0pt,outer sep=0pt,draw=figLine,line width=0.35pt,rounded corners=0.8pt,fill=white},
      collapsearrow/.style={->,figOrange,line width=0.85pt,shorten >=2.4pt,shorten <=2.4pt},
      collapsepath/.style={figOrange,line width=0.72pt},
    ]

    \node[title,anchor=west] at (0.00,5.58) {A. Random target-cell pairing};
    \node[branch,text=figOrange,anchor=west] at (0.10,4.82) {Target-action cell};
    \node[card] (clipa) at (1.00,4.08)
      {\includegraphics[width=2.50cm]{figures/motion_mesh_assets/figure2_prototype/ostrich_target_attack.png}};
    \node[card] (clipb) at (5.25,4.08)
      {\includegraphics[width=2.50cm]{figures/motion_mesh_assets/figure2_prototype/ostrich_target_attack2.png}};
    \node[card] (clipc) at (1.00,2.32)
      {\includegraphics[width=2.50cm]{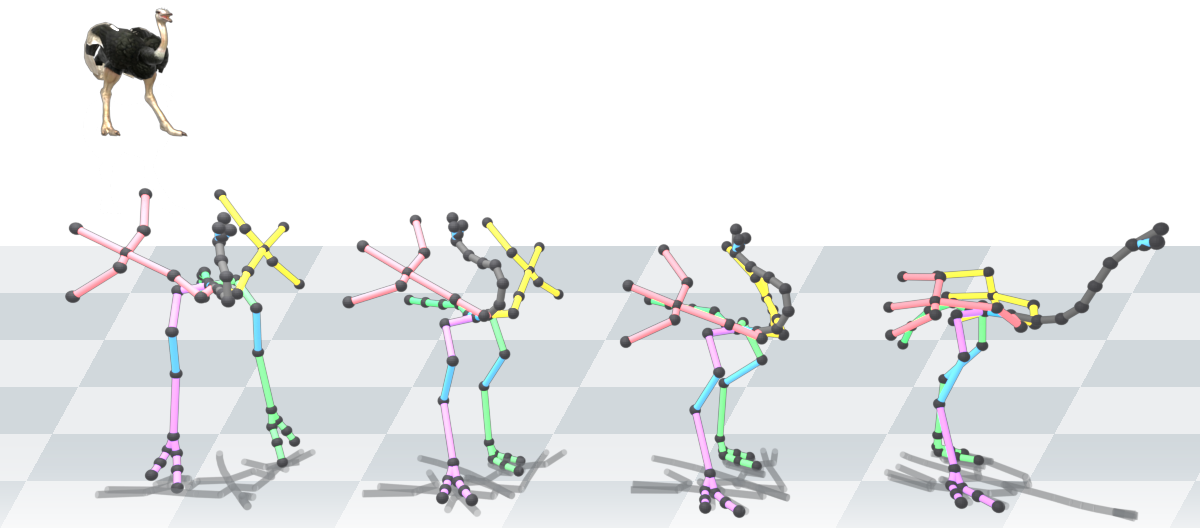}};
    \node[card] (clipd) at (5.25,2.32)
      {\includegraphics[width=2.50cm]{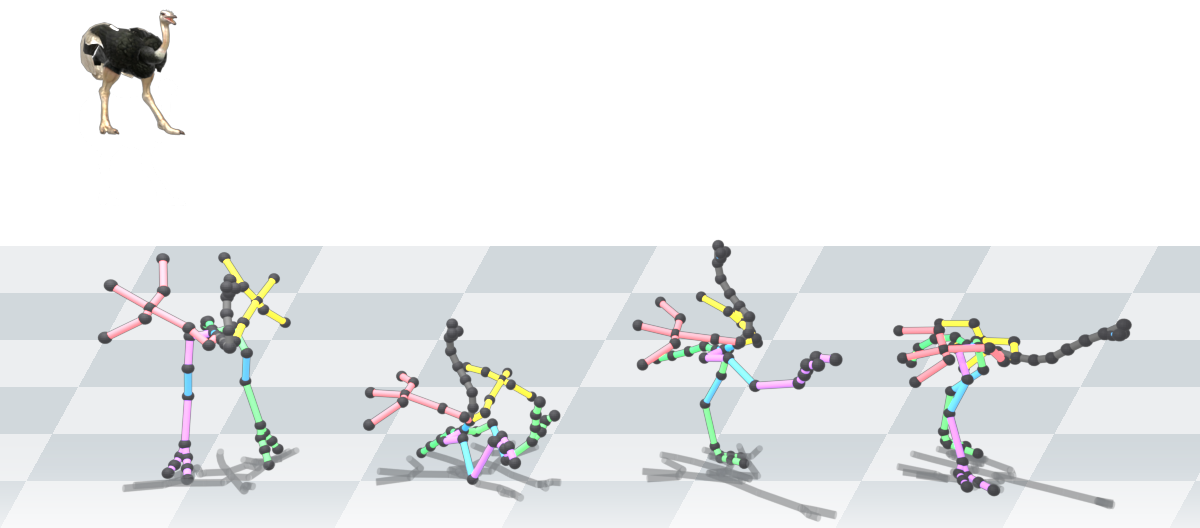}};

    \node[small,draw=figLine,rounded corners=1.2pt,fill=white,inner xsep=3.0pt,inner ysep=2.0pt,text=figInk]
      (loss) at (3.12,3.20) {$\min_f\,\mathbb{E}\|f(x_a,b)-x_b\|_2^2$};
    \node[small,align=center,text=figMuted] at (3.12,1.56) {source instance is not paired};

    \node[branch,text=figOrange,anchor=center] at (8.00,4.82) {Cell mean prototype};
    \node[card] (protofig) at (8.00,3.20)
      {\includegraphics[width=2.98cm]{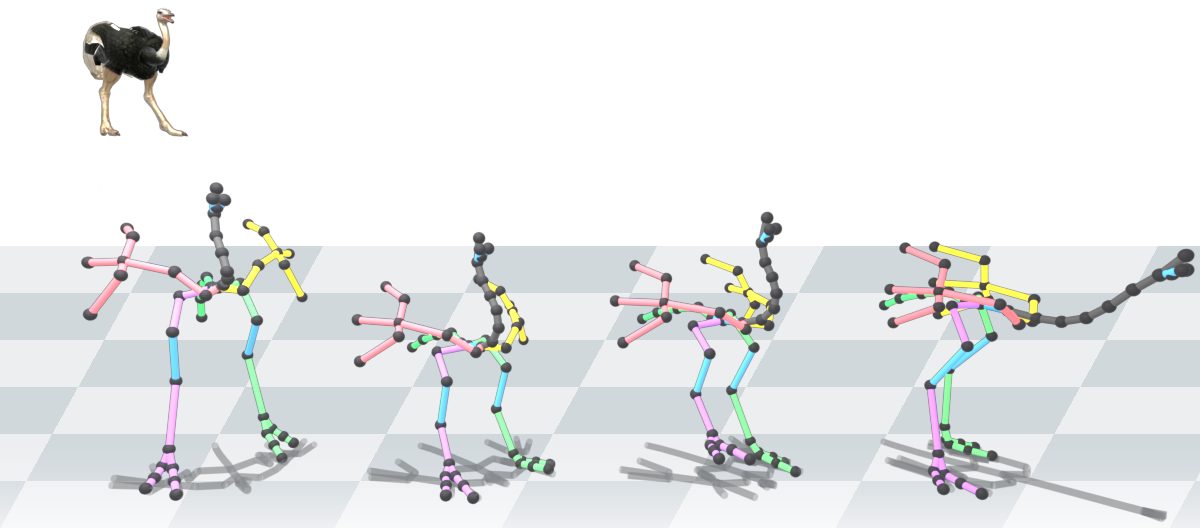}};

    \draw[collapsearrow,line width=1.05pt] (loss.east) -- (protofig.west);

    \draw[figLine,line width=0.90pt,densely dashed] (10.14,0.58) -- (10.14,5.42);

    \node[title,anchor=west] at (10.55,5.58) {B. Conditional-mean degeneration};
    \begin{scope}[xshift=0.35cm]
    \coordinate (axisO) at (11.22,1.05);
    \coordinate (axisX) at (15.96,1.05);
    \coordinate (axisY) at (11.22,4.58);
    \draw[figLine,line width=0.55pt] (axisO) -- (axisX);
    \draw[figInk,line width=0.65pt] (axisO) -- (axisY);

    \draw[figMuted,line width=0.55pt,densely dashed] (11.22,4.30) -- (15.96,4.30);
    \draw[figLine,line width=0.45pt,densely dotted] (11.22,1.05) -- (15.96,1.05);
    \node[small,anchor=west,text=figInk] at (14.90,4.43) {Oracle};
    \node[small,anchor=west,text=figMuted] at (11.26,1.40) {Cell mean};

    \draw[figLine,line width=0.38pt] (11.22,2.67) -- (15.96,2.67);
    \draw[figLine,line width=0.38pt] (11.22,1.05) -- (11.10,1.05);
    \draw[figLine,line width=0.38pt] (11.22,2.67) -- (11.10,2.67);
    \draw[figLine,line width=0.38pt] (11.22,4.30) -- (11.10,4.30);
    \node[small,anchor=east,text=figMuted] at (10.98,1.05) {0};
    \node[small,anchor=east,text=figMuted] at (10.98,2.67) {.5};
    \node[small,anchor=east,text=figMuted] at (10.98,4.30) {1.0};

    \foreach \x/\lab in {12.16/2,12.80/4,13.44/8,14.08/16,14.72/32,15.36/50} {
      \draw[figLine,line width=0.38pt] (\x,1.05) -- (\x,0.93);
      \node[small,anchor=north,text=figMuted] at (\x,0.83) {\lab};
    }

    \path[fill=figGreenFill,opacity=0.72]
      (12.16,1.134) -- (12.80,1.103) -- (13.44,1.132) -- (14.08,1.105) --
      (14.72,1.108) -- (15.36,1.111) -- (15.36,1.094) -- (14.72,1.099) --
      (14.08,1.092) -- (13.44,1.096) -- (12.80,1.091) -- (12.16,1.095) -- cycle;
    \draw[figGreen,line width=0.95pt]
      (12.16,1.115) -- (12.80,1.097) -- (13.44,1.114) -- (14.08,1.099) --
      (14.72,1.104) -- (15.36,1.103);
    \foreach \p in {(12.16,1.115),(12.80,1.097),(13.44,1.114),(14.08,1.099),(14.72,1.104),(15.36,1.103)} {
      \fill[figGreen,draw=white,line width=0.45pt] \p circle (2.6pt);
    }
    \node[small,anchor=west,align=left,text=figGreen] at (13.30,1.62)
      {Squared-error regressor,\\random same-cell targets};

    \node[axislabel,anchor=north] at (13.59,0.38) {$M$ pairs per cell};
    \node[axislabel,rotate=90,anchor=south] at (10.50,2.82)
      {Predicted / oracle variance};
    \end{scope}
  \end{tikzpicture}}
  \caption{Cell-paired squared-error regression collapses toward the target-cell
  prototype. Panel A illustrates random target-cell pairing: several admissible
  target clips define a cell, but the squared-error optimum is the cell mean
  prototype rather than a source-conditioned motion. Panel B reports the
  controlled ladder with known source-conditioned transport $T^\star$; from
  $M=2$ to $M=50$ pairs per cell, predicted output variance stays at about
  2\% of the oracle within-cell variance
  $\mathrm{Var}(T^\star(x_a))$. The shaded band is the three-seed standard
  deviation.}
  \label{fig:propcm_ladder}
\end{figure}

These two mechanisms also explain why retrieval comparators can be useful
diagnostics without identifying the retargeting map. The following formal check
records a boundary condition: a rule whose ranking depends only on features
computed from joint positions, rather than from a learned shared latent, is
outside the gauge variables of the objective above. It can therefore
characterize the source-blind floor without selecting the source-conditioned
map.

\begin{corollary}[Retrieval is gauge-invariant]\label{cor:retr-gauge}
Let $r:X_a\times\{b\}\to L_b$ be a retrieval rule that maps a source to a clip
in the target-skeleton library by a function of a feature $\phi$ that is itself
gauge-invariant, for example kinematic descriptors computed directly from joint
positions rather than from a learned latent. Then $r$ is invariant under the
gauge group $G_\Omega$ of any shared-latent objective: replacing $E_a$ with
$g_a\circ E_a$ leaves $r(x_a,b)$ unchanged. Furthermore $r(x_a,b)$ is a single
existing clip rather than an average, so the conditional-mean operator of
Prop.~\ref{prop:cm} does not apply.
\end{corollary}

The corollary describes a measurement, not a method to recommend. Its
invariance holds because the score is computed from the motion itself rather
than from learned latent coordinates.
We later instantiate this measurement role through label-only baselines and
ANCHOR, our deterministic target-library retrieval comparator. These comparators
can mark what action-level evidence accepts at the source-blind floor precisely
because they avoid the two objective mechanisms analyzed here. They do not
identify the retargeting map; they expose why standard evidence alone cannot.

\section{Source-Instance Fidelity and Empirical Evidence}
\label{sec:sif}

What the action-level score leaves out is whether differences among source
clips survive retargeting. A source-conditioned map should not merely produce a plausible
target motion; after the source skeleton, target skeleton, and action are fixed,
it should preserve the relative geometry among source clips. We therefore
define \emph{Source-Instance Fidelity} (SIF) on each eligible triple
$(a,b,c)$. For source clips $x_i,x_j\in\mathcal{X}_a$ and generated target
motions $y_i,y_j\in\mathcal{X}_b$, each pair is cropped to its common duration
and put into a common comparison frame by least-squares Procrustes
superposition, which removes global translation, rotation, and isotropic scale
\citep{Goodall1991Procrustes}. Let $d_{\mathrm{src}}(i,j)$ and
$d_{\mathrm{out}}(i,j)$ be the aligned per-frame discrepancies on the source
and output sides. Using Pearson's product-moment correlation
\citep{Pearson1896Regression}, we define
\begin{equation}
  \rho_{\textsc{sif}}
  =
  \mathrm{corr}\left(
  \{d_{\mathrm{src}}(i,j)\}_{i<j},
  \{d_{\mathrm{out}}(i,j)\}_{i<j}
  \right),
  \label{eq:sif}
\end{equation}
within each source-target-action triple and report the macro-average over
triples. Because the same source clips recur across triples, confidence
intervals resample whole source skeletons. We score every method twice: in raw
form, as above, and in length-controlled form, where every motion is first
resampled to 64 frames so that clip duration carries no information. Beside the
correlation we report the output variation, the average distance between outputs
divided by the average distance between their source clips; it is near one when
outputs vary as much as their sources and zero when they collapse to a single
motion.

SIF asks a different question from action retrieval. A generator can recover the
right target action by returning a target-action habit, in which case the
action-level score may be positive while $\rho_{\textsc{sif}}$ remains near
zero. A source-preserving map should instead make larger source-side changes
appear as larger target-side changes after target skeleton and action are fixed.
For failure-mode analysis, we also compute \emph{latent Source-Instance
Fidelity} (L-SIF) before decoding, replacing output distances with Frobenius
distances between the corresponding target latents
\citep{GolubVanLoan1996MatrixComputations}. This locates whether source-instance
information is absent in the latent representation or erased by the decoder.

The synthetic calibration fixes the scale of the diagnostic. In the controlled
2$\times$2 setting, the source-conditioned transport $T^\star$ is known for
every source clip, providing the instance-level correspondence absent from the
real corpus. Appendix~\ref{app:synth} shows that the oracle sits near one, while
two source-blind references are centered near zero with upper bootstrap bounds
below $+0.06$. On real data, a generator given no source at generation, with
output length fixed, scores exactly zero (Appendix~\ref{app:ace_controls}). To
decide whether a method sits at this floor, we use a shuffle test: within each
triple, outputs are reassigned at random among the source clips, the true
pairing included. A method is at the source-blind floor when its SIF does not
stand out from these shuffles under both raw and length-controlled scoring, and
near the floor when it does but stays inside the
$\abs{\rho_{\textsc{sif}}}<0.10$ band set by the synthetic references.

On the real corpus, SIF must be reported with the support that sparse data
allow. We evaluate every method on the 1,891 triples of
Section~\ref{sec:map}; methods that cannot produce an output for a triple are
scored on the remainder, which never falls below 1,413 triples. The original
evaluation, which pairs each source group with a single target (49 triples), is
reported in Appendix~\ref{app:sif_robustness}. The suffixes T and I mark
models trained on all 70 skeletons or on 60, with the other 10 held out as
targets. The comparison is organized by evidentiary
role rather than by rank: published generative objectives are evaluated beside
rows that isolate sparse correspondence, supervision-matched labels, latent
transport, target-library retrieval, and label-only sampling
\citep{raab2025anytop,moreflow2025,ace2023,m2m2024,DeBortoliEtAl2021DSB,unsb2023}.
Appendix~\ref{app:data_methods} records source-conditioning adaptations and
split conventions.

ANCHOR occupies a distinct diagnostic role: it is not a retargeting generator.
Given a source clip, target skeleton, and action request, it predicts a coarse
source action cluster from motion-space descriptors and retrieves an existing
target-library clip using that prediction, exact action metadata, and direct
kinematic descriptors. Because this rule sits outside the learned latent
objectives analyzed in Section~\ref{sec:structural}, it can show when
action-level evidence and source-conditioned transfer separate. The scoring rule
appears in Appendix~\ref{app:anchor}.

\begin{figure}[t]
  \centering
  \includegraphics[width=0.92\linewidth]{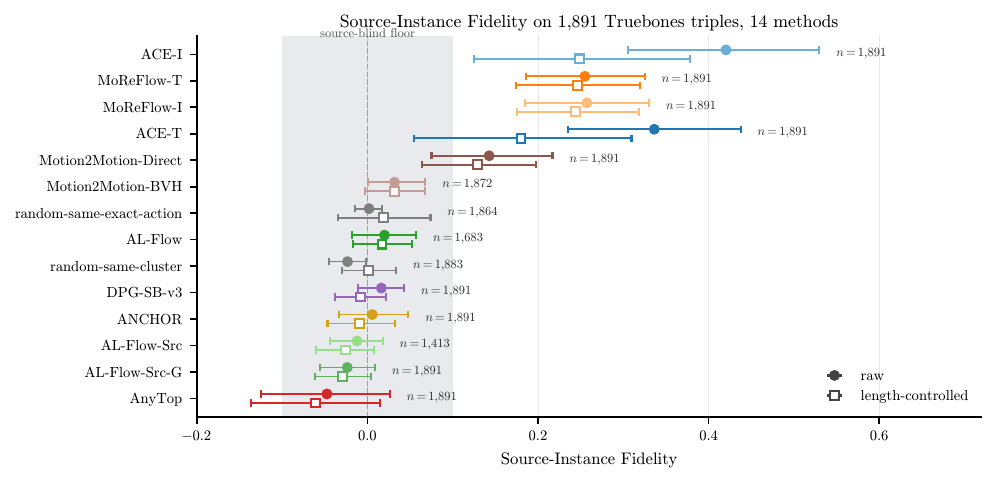}
  \caption{Source-Instance Fidelity for the fourteen evaluated methods on the
  1,891 Truebones triples. Filled circles use raw scoring and open squares
  length-controlled scoring; intervals are 95\% bootstrap intervals that
  resample whole source skeletons, and $n$ is the number of triples a method
  supports. The shaded band marks $\abs{\rho_{\textsc{sif}}}<0.10$. Eight
  methods do not stand out from shuffled pairings under either scoring, and
  Motion2Motion-BVH stays near zero.}
  \label{fig:sif_max_support}
\end{figure}

The main empirical pattern is that action-consistent generation usually does not
preserve source-instance geometry. Figure~\ref{fig:sif_max_support} shows that
nine of the fourteen evaluated methods sit at or near the source-blind floor
under both raw and length-controlled scoring: AnyTop, the three AL-Flow
variants, DPG-SB-v3, ANCHOR, and both random label-only references do not stand
out from shuffled pairings, and Motion2Motion-BVH stays near zero ($+0.032$).
Five methods rise above the floor. Motion2Motion-Direct keeps a modest
correlation ($+0.14$ raw, $+0.13$ length-controlled). MoReFlow-T and MoReFlow-I
remain above the floor at about $+0.25$ under both scorings, and their outputs
keep 16 to 21 percent of the source variation. MoReFlow trains on source-target
pairs matched by motion descriptors computed from joint positions, a coarse form
of the correspondence evidence the sparse regime lacks; we did not isolate this
matching step. ACE-T and ACE-I score higher in raw form ($+0.34$ and $+0.42$)
but drop under length control ($+0.18$ and $+0.25$), and their outputs vary by
only 2 to 3 percent of the source variation in raw scoring and by less than 1
percent when the length is fixed at generation. The result is consistent with
the structural claim in Section~\ref{sec:structural}: action-level evidence can be present while the
source-conditioned map remains unidentified. Appendix~\ref{app:sif_robustness}
reports the corresponding support, bootstrap, and latent-SIF checks.

The qualitative panorama makes the same ambiguity visible. In
Figure~\ref{fig:qual_method_panorama}, a single Bird $\to$ KingCobra attack
query is held fixed while the evaluated methods produce visibly different
target motions. The figure is not a substitute for SIF, but it shows why a
visual action judgment alone cannot identify the source-conditioned map: many
outputs can look action-consistent while corresponding to different maps.

\begin{figure}[t]
  \centering
  \includegraphics[width=1\linewidth]{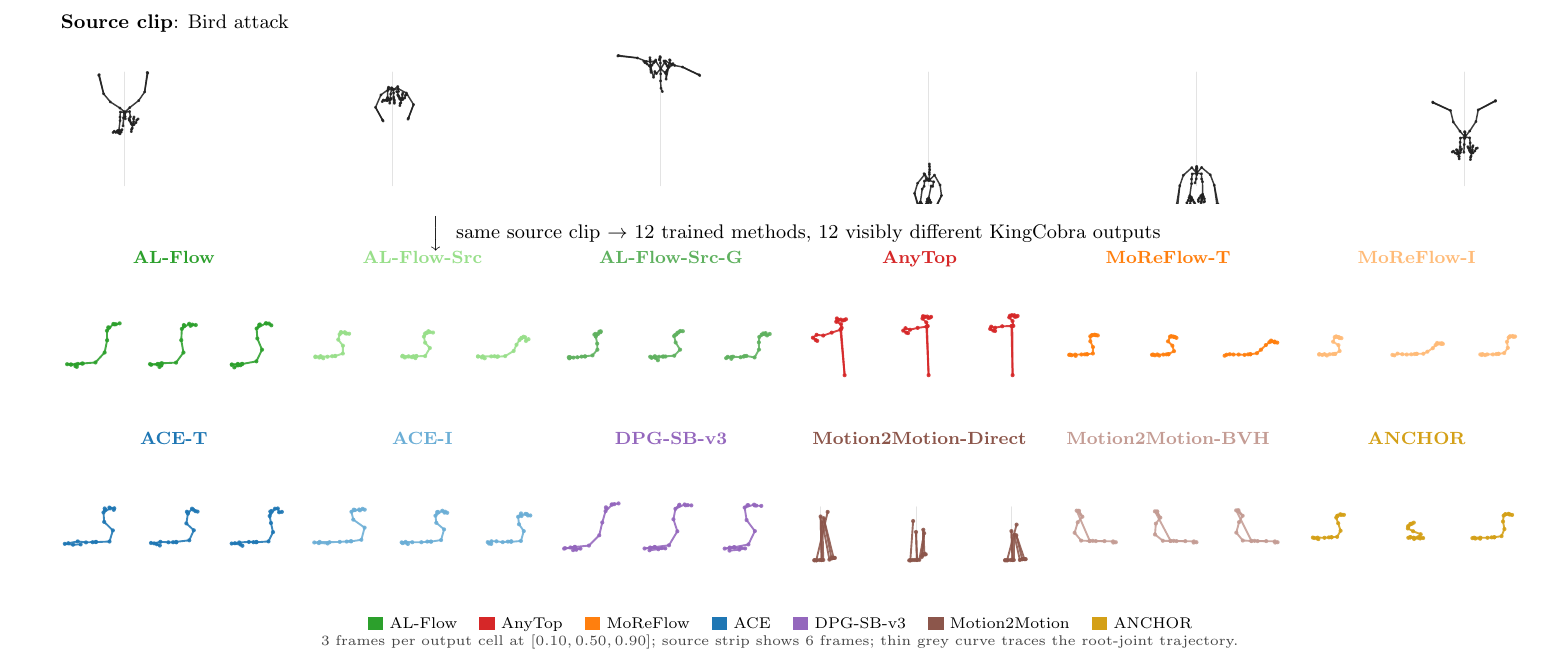}
  \caption{Qualitative non-identifiability panorama. A single Bird $\to$
  KingCobra attack query is shown across evaluated methods. All rows receive the
  same source clip, target skeleton, and action request, yet the generated
  target motions differ visibly across methods. The figure is qualitative
  evidence for the ambiguity isolated by SIF: action-consistent target motion
  can be produced without identifying a unique source-conditioned map.}
  \label{fig:qual_method_panorama}
\end{figure}

The panorama does not ask the reader to choose a visually best method,
nor does it treat a single attractive output as evidence of transfer. It fixes
the query and shows the range of outputs that remain possible once the action is
recognizable. SIF then supplies the missing relational
test: if the source clip changes within the same source-target-action cell, the
target output should change in the corresponding way rather than merely remain a
plausible KingCobra attack.

Action-level AUC rewards any output with a recognizable target action,
including outputs that SIF places at the source-blind floor.
Table~\ref{tab:label_only} gives the numerical audit, and
Figure~\ref{fig:label_only_audit} visualizes the same comparison. The
random-same-cluster reference chooses a random target clip with the same coarse
action cluster, while random-same-exact-action chooses a random target clip
with the exact action label whenever such a clip exists. These rules have no
access to a source-conditioned retargeting map, yet random-same-cluster matches
ANCHOR within $\pm 0.012$ on both the overall and held-out splits. The
exact-action variant exceeds ANCHOR by $+0.126$ overall and $+0.242$ held-out,
but only on the $221/600$ overall and $57/200$ held-out coverage where
exact-action positives exist in the target library. The qualifier is part of
the claim: strong action-level evidence can be generated by label availability
alone, without identifying the map. Appendix~\ref{app:enumeration} gives the
broader enumeration and action-level comparison tables behind this audit.

\begin{table}[!ht]
  \caption{Label-only retrieval can match or exceed ANCHOR's action-level AUC
  without solving the retargeting problem. AUC is measured under the same
  cluster-tier Procrustes protocol used for ANCHOR. The exact-action reference
  has restricted coverage, so its larger gaps are only defined on the rows where
  the target library contains an exact-action candidate.}
  \label{tab:label_only}
  \centering
  \small
  \begin{tabular}{lrrll}
\toprule
Method & Overall $n / 600$ & Held-out $n / 200$ & Overall AUC & Held-out AUC \\
\midrule
ANCHOR & 600 & 200 & 0.757 & 0.681 \\
random-same-cluster & 570 & 177 & 0.769\,(+0.012) & 0.678\,(-0.003) \\
random-same-exact-action & 221 & 57 & 0.883\,(+0.126) & 0.923\,(+0.242) \\
\midrule
\multicolumn{5}{p{0.94\linewidth}}{\footnotesize Parenthetical $\Delta$ values are method AUC minus ANCHOR. random-same-cluster matches ANCHOR within $\pm 0.012$ on both splits. random-same-exact-action exceeds ANCHOR by $+0.126$ overall and $+0.242$ held-out, but only on the $221/600$ and $57/200$ coverage where exact-action positives exist; the coverage qualifier travels with the gap.} \\
\bottomrule
\end{tabular}

\end{table}

\begin{figure}[!ht]
  \centering
  \includegraphics[width=0.54\linewidth]{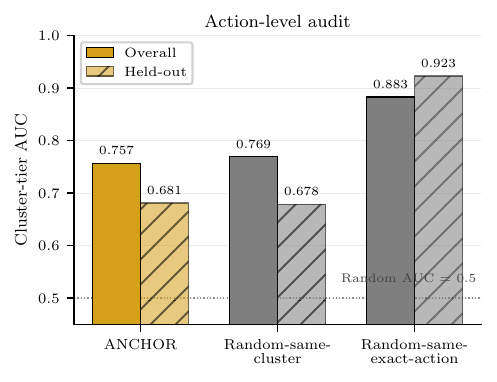}
  \caption{Action-level AUC cannot distinguish source-conditioned transport
  from label-based retrieval. ANCHOR retrieves from the target library using
  action information and motion descriptors, while the two random references use
  only label constraints. Their AUC values remain in the same action-level
  regime, showing that the standard score can accept target-action recovery even
  when source-instance transfer has not been established. The exact-action
  reference is coverage-limited to $221/600$ overall and $57/200$ held-out
  queries.}
  \label{fig:label_only_audit}
\end{figure}

The label-only audit changes how the positive SIF rows should be read. Once
action-level AUC can be reproduced by references with no source-conditioned map,
the question is not whether a method also scores well on action recovery, but
whether it contains independent motion-space evidence that preserves the source
instance after the target skeleton and action are fixed. The rows above the
floor answer this only in part. ACE retains a relational signal that enters
through its source input: with the source removed and the output length fixed at
generation, its SIF falls from $+0.29$ to exactly zero
(Appendix~\ref{app:ace_controls}). Its outputs, however, vary little with the
source, and under length control its correlation cannot be told apart from
MoReFlow's on the same triples.

The ACE controls in Appendix~\ref{app:ace_controls} also test where this signal
comes from. Removing the motion-space adversarial loss, in a comparison matched
in training data and seeds with the output length fixed at generation, leaves
SIF unchanged ($+0.33$ with and without the loss on the 1,891 triples, three
seeds each), so the adversarial constraint does not create the signal. The
nearest-neighbor leakage check in Table~\ref{tab:ace_leakage_nn} gives the
complementary control: $76$ to $79$ percent of ACE outputs are closer, under
coarse kinematic features, to a target-skeleton training clip than to the source
clip, which does not support a direct source-copying explanation. The structural
lesson is therefore specific: an auxiliary objective can leave a partial
source-instance signal without selecting a source-preserving map, and
action-level AUC by itself cannot tell the two apart.

The same pattern holds where true correspondences exist. We built two paired
human-to-robot evaluations with the SIF protocol unchanged
(Appendix~\ref{app:robots}). In the first, human motion clips from the
BONES-SEED dataset come with their retargeted counterparts on the Unitree G1
humanoid (90 action groups, 355 clips). In the second, we retargeted clips from
the LAFAN1 human motion dataset to six humanoid robots with the GMR
inverse-kinematics tool, so every clip has a counterpart on every robot. In
both, we trained one model in three ways: with an unpaired objective of the
Theorem~\ref{thm:gauge} class, with random same-action pairing as in
Proposition~\ref{prop:cm} (the averaging objective), and on the true pairs.
Table~\ref{tab:g1_joint} reports the G1 results on identical test items. A
randomly chosen same-action clip is recognized as the right action yet sits at
the floor, the two theorem objectives lose most of the source variation, and the
model trained on true pairs recovers both. Across the six robots, the averaging
objective collapsed in all 24 evaluations and the unpaired objective in 22,
while the true-pair model reached correlations of 0.77 to 0.99 in all 24. These
robots are humanoids, their pairs come from retargeting tools rather than
animators, and the sparsity is created by us.

\begin{table}[t]
  \caption{Human-to-G1 evaluation on identical test items (90 action groups,
  355 clips). SIF uses raw scoring; variation is the median ratio of output
  spread to source spread;
  realistic is the share of outputs whose nearest real G1 pose lies within the
  95th percentile of held-out real motion.}
  \label{tab:g1_joint}
  \centering
  \small
  \begin{tabular}{lrrrr}
    \toprule
    & SIF & Action AUC & Variation & Realistic \\
    \midrule
    True retargeted counterpart & $+0.960$ & 0.980 & 0.52 & 97.2\% \\
    Random same-action clip & $-0.008$ & 0.978 & 0.41 & 97.5\% \\
    Unpaired objective (Theorem~\ref{thm:gauge} class) & $+0.203$ & 0.499 & 0.001 & 100.0\% \\
    Averaging objective (Proposition~\ref{prop:cm} class) & $+0.379$ & 0.820 & 0.05 & 53.0\% \\
    Model trained on true pairs & $+0.900$ & 0.967 & 0.52 & 94.6\% \\
    \bottomrule
  \end{tabular}
\end{table}

\section{Limitations}
\label{sec:limitations}

SIF asks whether outputs are ordered like their sources. It does not measure how
much they differ, so a method can score well while its outputs barely change; we
therefore report output variation beside it. For a single group of three clips
the score is noisy, and it becomes reliable only as an average over many groups
(Appendix~\ref{app:sif_robustness}). A positive score can also appear without
any transfer of source motion. In our study, outputs that copied the length of
their source clip, a source latent left available at generation, and random
draws from small pools of same-action clips each produced one; length-controlled
scoring and the controls in Section~\ref{sec:sif} and
Appendix~\ref{app:ace_controls} separate these effects from source content. The animal evidence comes from a single dataset, and the enlarged
evaluation adds target bodies rather than new source motions. The robot studies
extend the picture to humanoids, but their paired motions come from retargeting
tools rather than animators, and their sparsity is created by us. Several
evaluated methods are also adaptations of published models to this data, and our
theorems cover specific training objectives rather than every method we
evaluate.

\section{Conclusion}
\label{sec:conclusion}

Cross-skeleton retargeting can look successful under action-level evaluation
even when the source-conditioned map has not been identified. We traced this to
two structural mechanisms. Under \emph{gauge non-identifiability}, matching each
body's motions only in distribution leaves the latent spaces of different bodies
free to be transformed relative to one another, so many maps fit the data. Under
\emph{conditional-mean degeneration}, pairing clips only by action drives
squared-error training to an average target motion instead of the motion
intention of the source clip. Source-Instance Fidelity makes the missing evidence
measurable by asking whether outputs differ the way their source clips do.

In sparse animal motion data, outputs that show the right action often sit at
the source-blind floor, and label-only references show that action-level AUC can
stay high without any map being selected. Methods that rise above the floor,
including ACE, retain only a partial relational signal, while on paired
human-to-robot data a model trained on true pairs recovers the source-instance
relation. Progress in retargeting should therefore be judged by whether an
objective improves action recovery, source-conditioned transport, or both. The
open question is which correspondence evidence or auxiliary objectives can select
a source-preserving map from the many maps that fit the same data.

\begin{ack}
This work was supported by the Research Council of Finland, Flagship program Finnish Center for Artificial Intelligence (FCAI), and the Research Council of Finland (357301, 358246). We acknowledge CSC – IT Center for Science, Finland, for awarding this project access to the LUMI supercomputer, owned by the EuroHPC Joint Undertaking, hosted by CSC (Finland) and the LUMI consortium through CSC. We acknowledge the computational resources provided by the Aalto Science-IT project.
\end{ack}

\bibliographystyle{plainnat}
\bibliography{references}


\appendix

\section{Extended Related Work}
\label{app:related}

Cross-skeleton retargeting sits between character animation, heterogeneous
motion generation, and representation learning. Earlier retargeting work
established that motion can be shared across bodies by making skeletal structure
explicit, by learning motion spaces intended to be less dependent on topology,
or by conditioning sequence models on skeleton context
\citep{Aberman2020SAN,same2024,SkeletonInContext2023}. The generative setting
studied here inherits this ambition, but places it in sparse heterogeneous
libraries where dense source-target correspondences are absent
\citep{raab2025anytop,moreflow2025,ace2023,m2m2024}. Our distinction is not
that these methods fail to synthesize plausible target motion. It is that
standard evidence can accept a target motion whose action is correct while the
source-conditioned map remains unidentified.

Recent work on broader object and animal motion further clarifies why this
identification question matters. Together, these papers move motion synthesis
away from one canonical human body and toward settings in which body identity,
morphology, and topology become part of the generative problem itself
\citep{raab2024single,lee2025how,egan2024dogcode,chen2025topologyagnostic}.
These works broaden the range of bodies that motion models must handle. The
structural question in this paper is complementary: in sparse heterogeneous
data, which evidence selects the motion intention carried from the source clip
rather than a target-side action habit?

The formal argument is closest in spirit to identifiability limits for
unsupervised representation learning. Locatello et al.
\citep{LocatelloEtAl2019DisentanglementImpossible} show that unsupervised
disentanglement cannot be guaranteed without inductive bias, while the
iVAE analysis of Khemakhem et al.~\citep{khemakhem2020ivae} shows how suitable
auxiliary variables can make nonlinear independent components identifiable. Our
gauge theorem is narrower than either statement. It does not rule out
identification with stronger side information; it shows that per-skeleton
marginal matching does not select the relative latent alignment required by a
source-preserving retargeting map.

Retrieval and motion-library methods provide a useful foil because they operate
directly in motion space. Motion graphs \citep{KovarGleicher2002MotionGraphs}
already exploit the idea that a library can preserve physically realized motion
snippets, while sparse-correspondence retrieval shows how such libraries can be
used when source and target bodies do not share topology \citep{m2m2024}. These
methods motivate the diagnostic comparison used in the main text: retrieval can
characterize what action-level evidence accepts, but by itself it does not
identify the source-conditioned retargeting map.

Schr\"odinger bridge and flow objectives offer another route through the same
space of ambiguities. Diffusion Schr\"odinger bridges
\citep{DeBortoliEtAl2021DSB} and unpaired neural Schr\"odinger bridges
\citep{unsb2023} motivate transport between distributions, and flow matching is
one way modern retargeting methods instantiate this distributional view
\citep{moreflow2025}. In the sparse cross-skeleton setting, distributional
transport alone does not remove the relative-gauge ambiguity. This is why the
paper separates action-level transport evidence from the source-instance
evidence measured by SIF.

\section{Data Regime and Method Roster}
\label{app:data_methods}

The structural limits studied in the main paper arise in a sparse heterogeneous
motion library. Table~\ref{tab:truebones_regime} records the Truebones regime
used throughout the experiments. The important fact is not merely that the
library is small, but that most skeleton-action cells are empty and most
occupied cells contain a single clip, making instance-level correspondence rare
even when action labels are available.

\begin{table}[t]
  \centering
  \caption{Truebones zoo regime. SIF is evaluated on source groups with at least
  three clips; the main evaluation pairs each group with every target skeleton
  that performs the same action, and the original evaluation pairs it with a
  single target.}
  \label{tab:truebones_regime}
  \begin{tabular}{lr}
\toprule
Quantity & Value \\
\midrule
Skeletons & 70 \\
Total clips & 616 \\
Coarse clusters & 10 \\
Exact actions & 90 \\
Skeleton-action cells & 6{,}300 ($70 \times 90$) \\
Occupied cells & 413 (6.6\%) \\
Median clips per occupied cell & 1 \\
Source groups with at least three clips & 49 (171 source clips) \\
SIF triples, every eligible target (main evaluation) & 1{,}891 (6{,}611 queries) \\
SIF triples, one target per group (original evaluation) & 49 (171 queries) \\
Original cross-method intersection & 37 triples (130 queries) \\
\bottomrule
\end{tabular}

\end{table}

The evaluated roster in Table~\ref{tab:method_roster} separates published
generative models, supervision-matched comparators, the deterministic ANCHOR
comparator, and source-blind random references. The public method names in this
table are the only names used for result interpretation.

\begin{table}[t]
  \centering
  \caption{Method roster. Train-skeleton counts distinguish transductive
  evaluation on all skeletons from inductive evaluation with held-out target
  skeletons.}
  \label{tab:method_roster}
  \resizebox{\linewidth}{!}{\begin{tabular}{p{0.34\linewidth}p{0.46\linewidth}r}
\toprule
Method & Family & Train skeletons \\
\midrule
AnyTop & source-conditioned diffusion & 70 \\
ACE-T & motion-space adversarial & 70 \\
ACE-I & motion-space adversarial & 60 \\
MoReFlow-T & flow matching & 70 \\
MoReFlow-I & flow matching & 60 \\
AL-Flow & label-conditional flow, no source motion & 60 \\
AL-Flow-Src & label-conditional flow + source motion + skeleton id & 60 \\
AL-Flow-Src-G & label-conditional flow + source motion + skeleton graph & 60 \\
Motion2Motion-Direct & patch retrieval and blending (training-free) & 60 \\
Motion2Motion-BVH & patch retrieval and blending (training-free) & 60 \\
DPG-SB-v3 & latent Schr\"odinger bridge & 60 \\
ANCHOR & retrieval comparator & 60 \\
random-same-cluster & trivial baseline & -- \\
random-same-exact-action & trivial baseline & -- \\
\midrule
\multicolumn{3}{p{0.94\linewidth}}{\footnotesize The AnyTop row uses the source-conditioned extension required by SIF; the original generator is unconditional with respect to a source clip. AL-Flow denotes the label-conditional flow family. ANCHOR is a deterministic retrieval comparator, not a retargeting solution. DPG-SB-v3 is evaluated in its latent-bridge form without decoded-motion penalties (see Appendix~\ref{app:dpgsb_status}).} \\
\bottomrule
\end{tabular}
}
\end{table}

\paragraph{Source-conditioning adaptations.}
Two rows require additional care because their original public form does not
match the SIF input contract. AnyTop is unconditional with respect to a source
clip, so the evaluated row adds a self-supervised source-motion encoder with
four learned motion queries and injects the resulting latent through
cross-attention into the target-skeleton decoder; training remains unpaired
self-reconstruction on Truebones. Motion2Motion is a sparse-correspondence
retrieval-and-blending procedure rather than a learned generator, so we report
both Motion2Motion-Direct, a reimplementation on the 13-channel representation,
and Motion2Motion-BVH, the official BVH implementation, using auto-heuristic
correspondences and excluding evaluator candidate clips from the target example
pool. These details affect only how each row receives source information; SIF is
computed by the same source-target-action protocol across rows.

\paragraph{Relation to the theory.}
Table~\ref{tab:theory_classes} places each evaluated method relative to the two
formal results. Theorem~\ref{thm:gauge} covers objectives that only match each
skeleton's motion distribution through a shared Gaussian latent, without paired
examples. Proposition~\ref{prop:cm} covers squared-error training against a
randomly chosen same-action target. No evaluated method instantiates either
class exactly, and the placement describes the structure of each objective
rather than a predicted SIF value.

\begin{table}[t]
  \centering
  \caption{Relation of each evaluated method to the formal results.}
  \label{tab:theory_classes}
  \small
  \begin{tabular}{p{0.36\linewidth}p{0.58\linewidth}}
\toprule
Method & Relation to the formal results \\
\midrule
AnyTop & closest to Theorem~\ref{thm:gauge}; adds rank and view-consistency terms \\
AL-Flow, AL-Flow-Src, AL-Flow-Src-G & closest to Proposition~\ref{prop:cm}; trained with flow matching rather than squared error \\
MoReFlow-T, MoReFlow-I & partly within Theorem~\ref{thm:gauge}; adds pairing by matched motion descriptors and a descriptor loss \\
DPG-SB-v3 & partly within Theorem~\ref{thm:gauge}; adds adversarial and cycle terms in the latent space \\
ACE-T, ACE-I & outside both classes; adversarial and source-feature losses \\
ANCHOR, random-same-cluster, random-same-exact-action & retrieval rules with no trained objective (Corollary~\ref{cor:retr-gauge}) \\
Motion2Motion-Direct, Motion2Motion-BVH & patch retrieval and blending with no trained objective \\
\bottomrule
\end{tabular}

\end{table}

\section{Proofs and Formal Details}
\label{app:proofs}

\subsection{Proof of Theorem~\ref{thm:gauge}}
\label{app:gauge_proof}

\begin{proof}
For each skeleton $s$ and motion $x\in\mathcal{X}_s$, the reparameterized
encoder-decoder pair satisfies
\[
  D'_s(E'_s(x))
  =D_s\!\left(g_s^{-1}(g_s(E_s(x)))\right)
  =D_s(E_s(x)),
\]
because $g_s$ is a measurable bijection. The per-skeleton reconstruction term
$L_s(D_s\circ E_s;P_s)$ is therefore unchanged pointwise.

The latent marginal induced by the reparameterized encoder is
$(E'_s)_\#P_s=(g_s\circ E_s)_\#P_s=(g_s)_\#\mu_s$. Since
$g\in G_\Omega(\mu)$, the marginal regularizer also remains unchanged:
\[
  \Omega((g_1)_\#\mu_1,\ldots,(g_K)_\#\mu_K)
  =\Omega(\mu_1,\ldots,\mu_K).
\]
Combining the unchanged reconstruction terms with the unchanged regularizer
gives $L(E',D')=L(E,D)$.

For the retargeting map, define
\[
  A=\{z: D_b(g_b^{-1}(g_a(z)))\neq D_b(z)\}.
\]
By the decoder non-degeneracy hypothesis in Theorem~\ref{thm:gauge}, this set
has positive $\mu_a$-measure for some pair $a,b$. Let
$B=E_a^{-1}(A)\subseteq\mathcal{X}_a$. Since $(E_a)_\#P_a=\mu_a$,
$P_a(B)=\mu_a(A)>0$. For every $x\in B$,
\[
  T'_{a\to b}(x)
  =D_b(g_b^{-1}(g_a(E_a(x))))
  \neq D_b(E_a(x))
  =T_{a\to b}(x),
\]
so the two retargeting maps differ on a set of positive $P_a$-measure. The
objective therefore determines the map only up to the relative gauge
$g_b^{-1}\circ g_a$.
\end{proof}

\begin{corollary}[Orthogonal-gauge lower bound under exact Gaussian alignment]
\label{cor:gauge_size}
Suppose $\Omega(\mu_1,\ldots,\mu_K)=
\sum_s\mathrm{KL}(\mu_s\,\|\,\mathcal{N}(0,I_d))$, or any
rotation-invariant function of the marginals, and each
$\mu_s=\mathcal{N}(0,I_d)$. Then $G_\Omega$ contains $O(d)^K$, the product of
orthogonal groups. The full product has $K\,d(d-1)/2$ continuous parameters.
After quotienting the diagonal $O(d)$ subgroup, one common orthogonal factor
applied to all skeletons, the relative gauge retains $(K-1)d(d-1)/2$
continuous degrees of freedom.
\end{corollary}

\begin{proof}
If $R\in O(d)$ and $z\sim\mathcal{N}(0,I_d)$, then
$Rz\sim\mathcal{N}(0,RR^\top)=\mathcal{N}(0,I_d)$. Thus every per-skeleton
choice $(R_1,\ldots,R_K)\in O(d)^K$ preserves all Gaussian-aligned marginals
and belongs to $G_\Omega$ whenever $\Omega$ is rotation-invariant. Since
$O(d)$ has dimension $d(d-1)/2$, the product group has $K\,d(d-1)/2$
continuous parameters. Quotienting the diagonal subgroup removes one $O(d)$
factor, because a common orthogonal transformation applied to all skeletons
cancels in the relative transformation $g_b^{-1}\circ g_a$, leaving
$(K-1)d(d-1)/2$ relative-gauge parameters. For the setting used in the paper,
$K=70$ and $d=32$, giving $69\cdot32\cdot31/2=34{,}224$ relative-gauge
degrees of freedom in this orthogonal subgroup. This is a lower bound on
relative-gauge degrees of freedom, not by itself a count of distinct decoded
retargeting maps; that latter count also depends on decoder non-degeneracy.
\end{proof}

\subsection{Proof of Proposition~\ref{prop:cm}}
\label{app:cm_proof}

\begin{proof}
Condition on a source clip $x_a$ and target skeleton $b$, and write
$c=\mathrm{action}(x_a)$. Under the proposition's sampling rule, the target
$x_b$ is drawn uniformly from the finite cell
$B(b,c)=\{x:(\mathrm{skel}(x),\mathrm{action}(x))=(b,c)\}$, independently of
the source instance once the action cell is fixed. For any proposed prediction
$u=f(x_a,b)$,
the sampling rule gives $x_b \perp x_a \mid (b,c)$, hence
$\mathbb{E}[x_b\mid x_a,b]=\mathbb{E}[x_b\mid b,c]$. Therefore
\[
  \mathbb{E}\|u-x_b\|_2^2
  =
  \|u-\mathbb{E}[x_b\mid b,c]\|_2^2
  +\mathbb{E}\|x_b-\mathbb{E}[x_b\mid b,c]\|_2^2 .
\]
The second term is independent of $u$, so the finite-cell minimizer is unique;
in the population version the same equality gives the almost surely unique
conditional-mean predictor $u=\mathbb{E}[x_b\mid b,c]$. In the empirical
finite-cell case, this conditional expectation is
\[
  |B(b,c)|^{-1}\sum_{x_b'\in B(b,c)}x_b' .
\]
It depends on the target skeleton and source action, but not on the specific
source clip. The Bayes predictor is therefore the target-action cell mean.
\end{proof}

\subsection{Proof of Corollary~\ref{cor:retr-gauge}}
\label{app:retr_gauge_proof}

\begin{proof}
The retrieval rule $r$ depends on the source motion through a feature
$\phi(x_a)$ that is computed directly from motion space and is invariant to the
latent-gauge transformations in $G_\Omega$. Replacing $E_a$ by $g_a\circ E_a$
therefore changes no argument supplied to $r$, and the returned library clip is
unchanged. Since $r(x_a,b)$ is a single existing target-skeleton clip rather
than an average over target-cell motions, the conditional-mean operator in
Proposition~\ref{prop:cm} is not the relevant limiting operation.
\end{proof}

\section{Decoder Non-Degeneracy Boundary}
\label{app:decoder_nondegeneracy}

The gauge theorem requires a decoder that can make at least part of the gauge
orbit observable. In a small Kullback-Leibler (KL)-regularized VAE, random
gauge rotations and matched-magnitude noise remain near the noise floor in
decoded effects, with a gauge-to-noise ratio of 1.11. This is the boundary case
in which the decoder nearly absorbs the orbit.

A separate AnyTop transformer diffusion hidden-state sweep tests decoder
observability, not the $d=32$ gauge-count model. Across hook layers, reverse
timesteps, source motions, and target skeletons, all 72 tested cells, formed by
four skeletons, two source clips, three hook layers, and three reverse
timesteps, give rotation-induced output divergence above the matched-noise
control, with a minimum gauge-to-noise ratio of 1.71. This check does not
instantiate the 34,224-dimensional lower-bound calculation. It shows only that
decoder observability can hold in a real transformer diffusion motion model,
leaving gauge-element selection as the structural problem.
Figure~\ref{fig:p3_perturbation} and Table~\ref{tab:p3_ratios} report the same
rotation-versus-noise check for the evaluated methods.

\begin{figure}[t]
  \centering
  \includegraphics[width=0.78\linewidth]{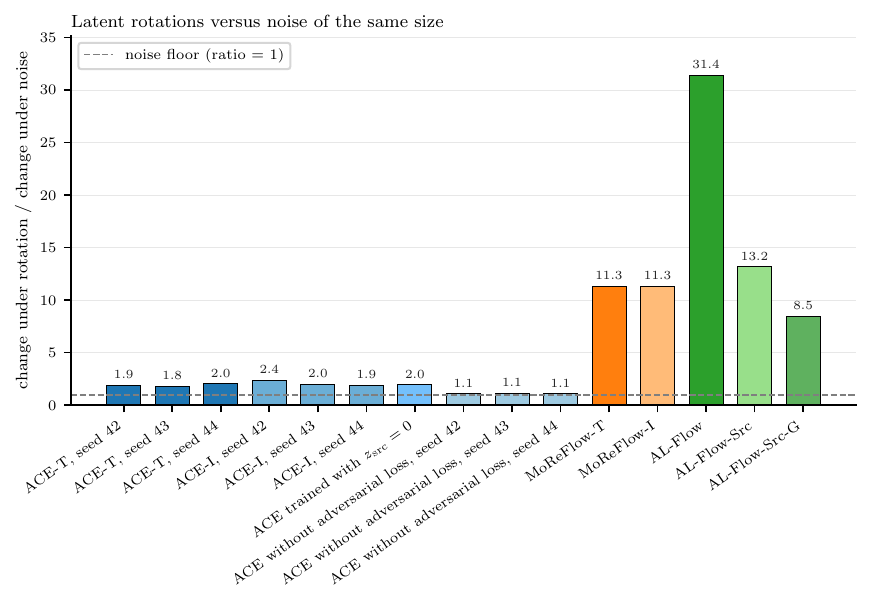}
  \caption{Decoder non-degeneracy check. Rotation-induced changes are compared
  with matched-magnitude noise in the latent space; ratios above one indicate
  that the decoder distinguishes at least some directions on the gauge orbit.}
  \label{fig:p3_perturbation}
\end{figure}

\begin{table}[t]
  \centering
  \caption{Rotation versus noise perturbations on the 130 queries of the
  37-triple intersection.
  The ACE ablation with the motion-space adversarial loss removed is the
  boundary case near the noise floor, while the other evaluated methods make the
  orbit visible after decoding.}
  \label{tab:p3_ratios}
  \resizebox{\linewidth}{!}{\begin{tabular}{llrrrl}
\toprule
Method & Seed & Change under rotation & Change under noise & Ratio & Verdict \\
\midrule
ACE-T (3 seeds) & 42 & 0.240 & 0.213 & 1.93 & non-degenerate \\
 & 43 & 0.219 & 0.183 & 1.78 & non-degenerate \\
 & 44 & 0.209 & 0.171 & 2.04 & non-degenerate \\
\midrule
ACE-I (3 seeds) & 42 & 0.121 & 0.089 & 2.35 & non-degenerate \\
 & 43 & 0.128 & 0.071 & 1.98 & non-degenerate \\
 & 44 & 0.141 & 0.087 & 1.92 & non-degenerate \\
\midrule
ACE trained with $z_{\text{src}}=0$ & 42 & 0.117 & 0.083 & 1.95 & non-degenerate \\
\midrule
ACE without adversarial loss (3 seeds) & 42 & 1.113 & 1.154 & \emph{1.08} & \emph{at noise floor} \\
 & 43 & 0.810 & 0.849 & \emph{1.12} & \emph{at noise floor} \\
 & 44 & 0.699 & 0.734 & \emph{1.10} & \emph{at noise floor} \\
\midrule
MoReFlow-T & -- & 0.113 & 0.017 & \textbf{11.33} & non-degenerate (strong) \\
\midrule
MoReFlow-I & -- & 0.107 & 0.016 & \textbf{11.32} & non-degenerate (strong) \\
\midrule
AL-Flow & -- & 0.080 & 0.008 & \textbf{31.40} & non-degenerate (strong) \\
\midrule
AL-Flow-Src & -- & 0.070 & 0.009 & \textbf{13.18} & non-degenerate (strong) \\
\midrule
AL-Flow-Src-G & -- & 0.050 & 0.008 & \textbf{8.46} & non-degenerate (strong) \\
\midrule
\multicolumn{6}{p{0.96\linewidth}}{\footnotesize Each row reports the mean over the 130 queries of the 37-triple intersection, with 20 random $O(d)$ rotations and 20 matched-magnitude noise samples per query; the ratio is the mean of the per-query rotation-to-noise ratios. ACE without the adversarial loss is statistically indistinguishable from the noise floor (Wilcoxon $p>0.4$ for all three seeds); every other method exceeds the noise floor at $p<10^{-12}$.} \\
\bottomrule
\end{tabular}
}
\end{table}

The cross-seed latent checks in Tables~\ref{tab:gauge_procrustes}
and~\ref{tab:effective_rank} give a complementary view. Orthogonal Procrustes
alignment \citep{Goodall1991Procrustes} explains a substantial fraction of
cross-seed latent variation, while effective rank separates rich latent
variation from low-dimensional collapse. These checks defend the boundary of
the theory: decoder observability can hold even when the objective does not
select the relative gauge.
Figures~\ref{fig:gauge_procrustes}, \ref{fig:effective_rank}, and
\ref{fig:p3_vs_sif_decoupling} visualize the same distinction between latent
alignment, representation richness, and source-instance fidelity.

\begin{table}[t]
  \centering
  \caption{Orthogonal Procrustes alignment across seeds. The table reports how
  much latent variation can be explained by an $O(d)$ alignment across the
  SIF-intersection queries.}
  \label{tab:gauge_procrustes}
  \begin{tabular}{llrr}
\toprule
Group & Seed pair & Share explained by rotation & Cosine after rotation \\
\midrule
ACE-T & 42 vs 43 & 0.522 & 0.893 \\
ACE-T & 42 vs 44 & 0.703 & 0.900 \\
ACE-T & 43 vs 44 & 0.585 & 0.905 \\
\midrule
ACE-I & 42 vs 43 & 0.689 & 0.935 \\
ACE-I & 42 vs 44 & 0.712 & 0.928 \\
ACE-I & 43 vs 44 & 0.775 & 0.941 \\
\midrule
ACE without adversarial loss & 42 vs 43 & 0.482 & 0.866 \\
ACE without adversarial loss & 42 vs 44 & 0.428 & 0.847 \\
ACE without adversarial loss & 43 vs 44 & 0.753 & 0.881 \\
\midrule
\multicolumn{4}{p{0.78\linewidth}}{\footnotesize \emph{Share explained by rotation} is the fraction of cross-seed latent variance accounted for by the best $O(d)$ rotation (orthogonal Procrustes). \emph{Cosine after rotation} is the mean per-row cosine similarity after that rotation. Both are computed on per-target-skeleton latents across the 130 queries of the 37-triple set, aggregated over 21 target skeletons.} \\
\bottomrule
\end{tabular}

\end{table}

\begin{figure}[t]
  \centering
  \includegraphics[width=0.78\linewidth]{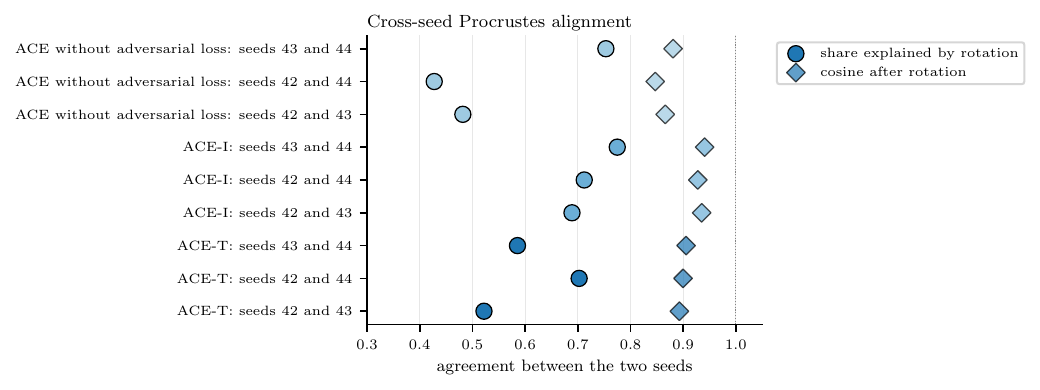}
  \caption{Cross-seed Procrustes alignment. High explained fractions indicate
  that latent spaces differ partly by approximately orthogonal transformations,
  matching the relative-gauge mechanism in Theorem~\ref{thm:gauge}.}
  \label{fig:gauge_procrustes}
\end{figure}

\begin{table}[t]
  \centering
  \caption{Effective-rank diagnostic on latent distributions. Higher effective
  rank indicates richer latent variation, but rank alone does not imply that the
  source-conditioned map has been identified.}
  \label{tab:effective_rank}
  {\small
\begin{tabular}{lrrr}
\toprule
Method & Sample shape & Spectral flatness & Effective rank $/$ 256 \\
\midrule
ACE without adversarial loss (seed 42) & 3{,}376 $\times$ 256 & 0.0001 & 1.29 \\
ACE-T (seed 42) & 3{,}376 $\times$ 256 & 0.0001 & 2.13 \\
ACE-I (seed 42) & 3{,}376 $\times$ 256 & 0.0019 & 3.00 \\
AL-Flow & 1{,}040 $\times$ 256 & 0.0066 & 6.80 \\
MoReFlow-I & 3{,}823 $\times$ 256 & 0.0593 & 12.84 \\
MoReFlow-T & 3{,}823 $\times$ 256 & 0.0631 & 14.99 \\
AnyTop encoder $z$ & 13{,}000 $\times$ 256 & 0.1614 & \textbf{57.73} \\
\midrule
\multicolumn{4}{p{0.85\linewidth}}{\footnotesize Effective rank computed on the per-method latent distribution over the 130 queries of the 37-triple intersection. AnyTop's encoder hidden state is an order of magnitude richer than the compressed latents shared by the MoReFlow-based models.} \\
\bottomrule
\end{tabular}
}

\end{table}

\begin{figure}[t]
  \centering
  \includegraphics[width=0.78\linewidth]{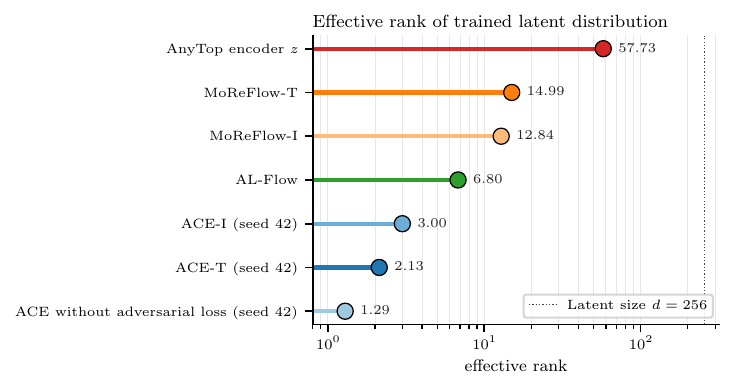}
  \caption{Effective rank across latent representations. The diagnostic
  separates representation richness from source-instance fidelity.}
  \label{fig:effective_rank}
\end{figure}

\begin{figure}[t]
  \centering
  \includegraphics[width=0.78\linewidth]{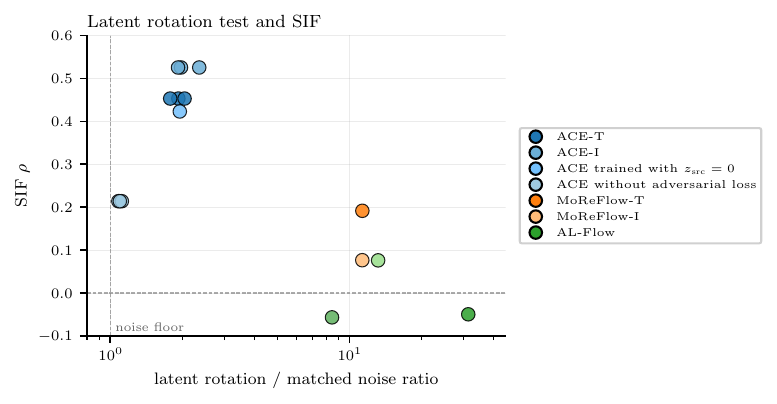}
  \caption{Decoder non-degeneracy and SIF are distinct. A method may expose
  latent perturbations after decoding while still lying near the source-blind
  floor under SIF.}
  \label{fig:p3_vs_sif_decoupling}
\end{figure}

\section{Source-Instance Fidelity Robustness}
\label{app:sif_robustness}

SIF is a diagnostic for whether source-side geometry remains visible after the
target skeleton and action are fixed. This appendix gives the numbers behind
Figure~\ref{fig:sif_max_support}, the original evaluation on 49 triples, and
checks on dependence, sample size, and generation noise.

Table~\ref{tab:sif_main} lists every method on the 1,891 triples of the main
evaluation. Each source clip appears in many triples there, so the shuffle test
moves all triples from the same source skeleton together, and the intervals
resample whole source skeletons. Treating the triples as independent would
overstate the precision. The last column applies the rule of
Section~\ref{sec:sif}.

\begin{table}[t]
  \centering
  \caption{SIF on the main evaluation (1,891 triples). Raw and length-controlled
  SIF with 95\% intervals that resample whole source skeletons; $p$ is the
  one-sided shuffle-test $p$-value; variation is the median ratio of output
  spread to source spread under raw scoring; $n$ is the number of triples a
  method supports. The last row is the ACE ablation of
  Appendix~\ref{app:ace_controls}, not an evaluated method.}
  \label{tab:sif_main}
  \resizebox{\linewidth}{!}{\begin{tabular}{lrlrlrrl}
\toprule
Method & $n$ & Raw SIF [95\% CI] & $p$ & Length-controlled SIF [95\% CI] & $p$ & Variation & Reading \\
\midrule
ACE-I & 1{,}891 & $+$0.420 [$+$0.31, $+$0.53] & $<$0.001 & $+$0.248 [$+$0.13, $+$0.38] & $<$0.001 & 0.03 & above floor \\
MoReFlow-T & 1{,}891 & $+$0.255 [$+$0.19, $+$0.33] & $<$0.001 & $+$0.246 [$+$0.17, $+$0.32] & $<$0.001 & 0.20 & above floor \\
MoReFlow-I & 1{,}891 & $+$0.257 [$+$0.19, $+$0.33] & $<$0.001 & $+$0.244 [$+$0.18, $+$0.32] & $<$0.001 & 0.16 & above floor \\
ACE-T & 1{,}891 & $+$0.336 [$+$0.24, $+$0.44] & $<$0.001 & $+$0.180 [$+$0.05, $+$0.31] & 0.009 & 0.02 & above floor \\
Motion2Motion-Direct & 1{,}891 & $+$0.143 [$+$0.08, $+$0.22] & $<$0.001 & $+$0.129 [$+$0.06, $+$0.20] & $<$0.001 & 1.52 & above floor \\
Motion2Motion-BVH & 1{,}872 & $+$0.032 [0.000, $+$0.07] & 0.024 & $+$0.032 [$-$0.003, $+$0.07] & 0.016 & 0.79 & near floor \\
random-same-exact-action & 1{,}864 & $+$0.002 [$-$0.01, $+$0.02] & 0.413 & $+$0.018 [$-$0.03, $+$0.07] & 0.293 & 0.13 & at floor \\
AL-Flow & 1{,}683 & $+$0.020 [$-$0.02, $+$0.06] & 0.114 & $+$0.017 [$-$0.02, $+$0.05] & 0.135 & 0.03 & at floor \\
random-same-cluster & 1{,}883 & $-$0.023 [$-$0.05, $-$0.002] & 0.958 & $+$0.001 [$-$0.03, $+$0.03] & 0.475 & 0.44 & at floor \\
DPG-SB-v3 & 1{,}891 & $+$0.016 [$-$0.01, $+$0.04] & 0.158 & $-$0.008 [$-$0.04, $+$0.02] & 0.700 & 0.22 & at floor \\
ANCHOR & 1{,}891 & $+$0.006 [$-$0.03, $+$0.05] & 0.387 & $-$0.010 [$-$0.05, $+$0.03] & 0.683 & 0.00 & at floor \\
AL-Flow-Src & 1{,}413 & $-$0.012 [$-$0.04, $+$0.02] & 0.744 & $-$0.026 [$-$0.06, $+$0.01] & 0.928 & 0.07 & at floor \\
AL-Flow-Src-G & 1{,}891 & $-$0.024 [$-$0.06, $+$0.01] & 0.931 & $-$0.029 [$-$0.06, $+$0.004] & 0.966 & 0.03 & at floor \\
AnyTop & 1{,}891 & $-$0.048 [$-$0.12, $+$0.03] & 0.915 & $-$0.061 [$-$0.14, $+$0.01] & 0.954 & 1.50 & at floor \\
\midrule
ACE trained with $z_{\text{src}}=0$ (real $z_{\text{src}}$ at generation) & 1{,}891 & $+$0.371 [$+$0.25, $+$0.49] & $<$0.001 & $+$0.198 [$+$0.07, $+$0.34] & 0.006 & 0.02 & ablation \\
\bottomrule
\end{tabular}
}
\end{table}

Table~\ref{tab:sif_original} reports the original evaluation, which pairs each of
the 49 source groups with a single target, so no source clip appears in two
triples. The pattern matches the main evaluation, with one difference:
Motion2Motion-Direct sits at the floor here and rises modestly above it only on
the larger set. The random exact-action reference is just significant in raw
scoring ($p=0.049$) and returns to the floor under length control. Its raw
signal comes from small candidate pools and clips cropped to their source's
length. Table~\ref{tab:clustered_ci} and Figure~\ref{fig:clustered_ci_robustness}
give raw-scoring intervals on the 37 triples that all methods support,
clustered by source skeleton, target skeleton, or action.

\begin{table}[t]
  \centering
  \caption{SIF on the original evaluation (49 triples, one target per source
  group). Columns as in Table~\ref{tab:sif_main}; here no source clip is shared
  between triples.}
  \label{tab:sif_original}
  \small
  \begin{tabular}{lrrrrrr}
\toprule
Method & $n$ & Raw SIF & $p$ & Length-controlled SIF & $p$ & Variation \\
\midrule
MoReFlow-T & 49 & $+$0.203 & 0.011 & $+$0.220 & 0.005 & 0.25 \\
ACE-T & 49 & $+$0.385 & $<$0.001 & $+$0.205 & 0.013 & 0.02 \\
ACE-I & 49 & $+$0.484 & $<$0.001 & $+$0.178 & 0.022 & 0.02 \\
MoReFlow-I & 49 & $+$0.137 & 0.061 & $+$0.160 & 0.038 & 0.28 \\
Motion2Motion-Direct & 49 & $+$0.053 & 0.290 & $+$0.091 & 0.146 & 1.69 \\
ANCHOR & 49 & $+$0.014 & 0.380 & $+$0.035 & 0.226 & 0.00 \\
random-same-exact-action & 48 & $+$0.096 & 0.049 & $+$0.032 & 0.338 & 0.00 \\
random-same-cluster & 48 & $+$0.073 & 0.170 & $+$0.008 & 0.464 & 0.25 \\
AL-Flow-Src & 38 & $+$0.070 & 0.258 & $-$0.009 & 0.534 & 0.11 \\
DPG-SB-v3 & 49 & $-$0.058 & 0.756 & $-$0.031 & 0.644 & 0.22 \\
AL-Flow & 45 & $-$0.066 & 0.769 & $-$0.044 & 0.682 & 0.04 \\
AnyTop & 49 & $-$0.033 & 0.639 & $-$0.065 & 0.773 & 1.34 \\
AL-Flow-Src-G & 49 & $-$0.028 & 0.629 & $-$0.104 & 0.882 & 0.08 \\
Motion2Motion-BVH & 49 & $-$0.085 & 0.823 & $-$0.171 & 0.982 & 0.93 \\
\bottomrule
\end{tabular}

\end{table}

\begin{table}[t]
  \centering
  \caption{Clustered bootstrap intervals for SIF. Clustering by source
  skeleton, target skeleton, or action tests whether the positive conclusions
  are driven by a narrow subset of triples.}
  \label{tab:clustered_ci}
  {\scriptsize
\begin{tabular}{lrrrr}
\toprule
Method & each triple alone & by source skeleton & by target skeleton & by action \\
\midrule
ACE-I & \textbf{0.526\,[0.351,\,0.681]} & \textbf{0.526\,[0.323,\,0.709]} & \textbf{0.526\,[0.365,\,0.671]} & \textbf{0.526\,[0.242,\,0.629]} \\
ACE-T & \textbf{0.453\,[0.293,\,0.602]} & \textbf{0.453\,[0.284,\,0.618]} & \textbf{0.453\,[0.274,\,0.605]} & 0.453\,[-0.009,\,0.619] \\
MoReFlow-T & 0.192\,[-0.003,\,0.380] & 0.192\,[-0.023,\,0.386] & 0.192\,[-0.031,\,0.405] & 0.192\,[-0.039,\,0.357] \\
random-same-exact-action & \textbf{0.185\,[0.046,\,0.326]} & \textbf{0.185\,[0.069,\,0.304]} & \textbf{0.185\,[0.027,\,0.354]} & 0.185\,[0.000,\,0.272] \\
Motion2Motion-Direct & 0.131\,[-0.061,\,0.317] & 0.131\,[-0.098,\,0.385] & 0.131\,[-0.075,\,0.307] & 0.131\,[-0.720,\,0.265] \\
random-same-cluster & 0.104\,[-0.095,\,0.298] & 0.104\,[-0.057,\,0.279] & 0.104\,[-0.115,\,0.315] & 0.104\,[-0.272,\,0.189] \\
AL-Flow-Src & 0.076\,[-0.108,\,0.264] & 0.076\,[-0.109,\,0.269] & 0.076\,[-0.112,\,0.248] & 0.076\,[-0.222,\,0.169] \\
MoReFlow-I & 0.076\,[-0.132,\,0.279] & 0.076\,[-0.132,\,0.271] & 0.076\,[-0.167,\,0.319] & 0.076\,[-0.108,\,0.474] \\
ANCHOR & 0.040\,[-0.068,\,0.157] & 0.040\,[-0.078,\,0.176] & 0.040\,[-0.068,\,0.167] & 0.040\,[0.000,\,0.046] \\
AnyTop & -0.010\,[-0.185,\,0.166] & -0.010\,[-0.191,\,0.166] & -0.010\,[-0.207,\,0.185] & -0.010\,[-0.253,\,0.142] \\
DPG-SB-v3 & -0.019\,[-0.215,\,0.178] & -0.019\,[-0.208,\,0.188] & -0.019\,[-0.244,\,0.228] & -0.019\,[-0.487,\,0.045] \\
AL-Flow & -0.050\,[-0.231,\,0.137] & -0.050\,[-0.187,\,0.106] & -0.050\,[-0.259,\,0.161] & -0.050\,[-0.239,\,0.183] \\
AL-Flow-Src-G & -0.057\,[-0.239,\,0.130] & -0.057\,[-0.207,\,0.087] & -0.057\,[-0.261,\,0.128] & -0.057\,[-0.523,\,0.022] \\
Motion2Motion-BVH & -0.067\,[-0.237,\,0.110] & -0.067\,[-0.247,\,0.116] & -0.067\,[-0.221,\,0.089] & -0.067\,[-0.320,\,0.142] \\
\midrule
\multicolumn{5}{p{0.96\linewidth}}{\footnotesize Bold lower CI excludes zero. Bootstrap $N=10{,}000$, seed $=42$. ACE-I is the only method whose lower CI excludes zero under all four resampling schemes; ACE-T and random-same-exact-action survive resampling by source skeleton and by target skeleton, but resampling by action drops their lower CI to or below zero because the 37-triple intersection is concentrated on a few actions.} \\
\bottomrule
\end{tabular}
}

\end{table}

\begin{figure}[t]
  \centering
  \includegraphics[width=0.80\linewidth]{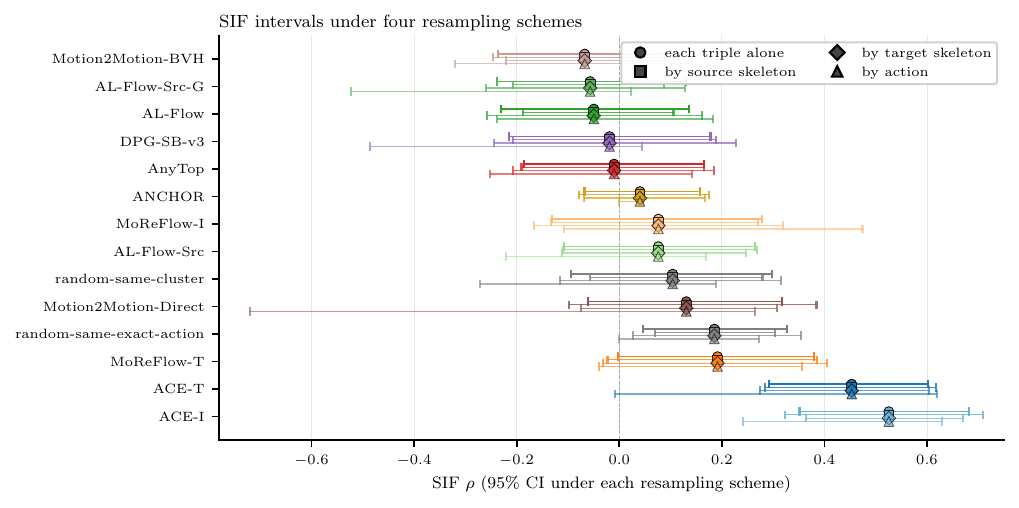}
  \caption{Clustered confidence intervals for SIF. The figure visualizes the
  dependence-aware intervals reported in Table~\ref{tab:clustered_ci}.}
  \label{fig:clustered_ci_robustness}
\end{figure}

Table~\ref{tab:paired_ace_moreflow} compares ACE with MoReFlow on identical
triples. In raw scoring ACE-I scores higher than MoReFlow-I. Under length
control, neither ACE model differs from MoReFlow in either evaluation.

\begin{table}[t]
  \centering
  \caption{Paired SIF differences on identical triples, with 95\% intervals that
  resample whole source skeletons and two-sided sign-flip $p$-values that flip
  all triples of a source skeleton together.}
  \label{tab:paired_ace_moreflow}
  \resizebox{\linewidth}{!}{\begin{tabular}{llll}
\toprule
Evaluation & Scoring & ACE-T minus MoReFlow-T & ACE-I minus MoReFlow-I \\
\midrule
1{,}891 triples & raw & $+$0.081 [$-$0.02, $+$0.19], $p$=0.139 & $+$0.163 [$+$0.05, $+$0.28], $p$=0.008 \\
1{,}891 triples & length-controlled & $-$0.066 [$-$0.19, $+$0.06], $p$=0.325 & $+$0.004 [$-$0.10, $+$0.11], $p$=0.935 \\
49 triples & raw & $+$0.182 [$-$0.06, $+$0.44], $p$=0.156 & $+$0.347 [$+$0.11, $+$0.60], $p$=0.009 \\
49 triples & length-controlled & $-$0.015 [$-$0.21, $+$0.19], $p$=0.888 & $+$0.018 [$-$0.21, $+$0.24], $p$=0.880 \\
\bottomrule
\end{tabular}
}
\end{table}

A single triple with three source clips gives a noisy estimate. To measure this,
we took the triples with five source clips and recomputed SIF on every
three-clip subset. Table~\ref{tab:sif_stability} shows that the value for one
triple moves with a standard deviation of about 0.6 across subsets. The average
over randomly drawn sets of triples moves by about 0.12 for 25 triples and by
about 0.02 for 500. Because the 1,891 triples reuse the same 171 source clips,
the uncertainty reported for the main evaluation comes from the intervals that
resample whole source skeletons (Table~\ref{tab:sif_main}).

\begin{table}[t]
  \centering
  \caption{Noise of a single triple and stability of the average. The first
  column is the standard deviation of SIF across three-clip subsets of
  five-clip triples; the other columns are standard deviations of the average
  over $k$ randomly drawn triples.}
  \label{tab:sif_stability}
  \small
  \begin{tabular}{lrrrrr}
\toprule
 & One triple & \multicolumn{4}{c}{Average over $k$ triples} \\
\cmidrule(lr){3-6}
Method & (three clips) & $k=25$ & $k=100$ & $k=500$ & $k=1{,}000$ \\
\midrule
AnyTop & 0.60 & 0.122 & 0.060 & 0.021 & 0.012 \\
ACE-T & 0.59 & 0.124 & 0.060 & 0.022 & 0.013 \\
ACE-I & 0.56 & 0.110 & 0.054 & 0.019 & 0.012 \\
MoReFlow-T & 0.57 & 0.123 & 0.059 & 0.023 & 0.013 \\
random-same-exact-action & 0.31 & 0.089 & 0.046 & 0.017 & 0.010 \\
\bottomrule
\end{tabular}

\end{table}

Methods that sample noise raise a further question: whether shared noise can
create an apparent signal. Table~\ref{tab:noise_control} regenerates every
stochastic method with three independent noise seeds, and once with a single
noise sample shared within each triple. With independent noise, every
stochastic method stays within 0.05 of zero on average. With shared noise,
AnyTop and the source-fed AL-Flow variants rise to about $+0.2$. The outputs then
differ only through their inputs, such as length, and SIF picks up that
difference without any source content. All results in the paper use independent
noise. MoReFlow is deterministic, so its score does not depend on noise.

\begin{table}[t]
  \centering
  \caption{Generation noise control on the original 49 triples. Independent
  noise uses three seeds; shared noise reuses one noise sample within each
  triple.}
  \label{tab:noise_control}
  \small
  \begin{tabular}{lrrrrr}
\toprule
 & \multicolumn{3}{c}{Independent noise, three seeds} & & \\
\cmidrule(lr){2-4}
Method & seed 1 & seed 2 & seed 3 & Mean $\pm$ sd & Shared noise \\
\midrule
AnyTop & $-$0.033 & $+$0.040 & $-$0.098 & $-$0.031 $\pm$ 0.069 & $+$0.243 \\
AL-Flow & $-$0.066 & $-$0.043 & $+$0.066 & $-$0.015 $\pm$ 0.070 & $+$0.002 \\
AL-Flow-Src & $+$0.070 & $+$0.049 & $-$0.140 & $-$0.007 $\pm$ 0.115 & $+$0.199 \\
AL-Flow-Src-G & $-$0.028 & $+$0.048 & $+$0.075 & $+$0.031 $\pm$ 0.053 & $+$0.239 \\
DPG-SB-v3 & $-$0.058 & $+$0.138 & $+$0.052 & $+$0.044 $\pm$ 0.098 & $+$0.064 \\
MoReFlow-T & $+$0.203 & $+$0.203 & $+$0.203 & $+$0.203 $\pm$ 0.000 & deterministic \\
MoReFlow-I & $+$0.137 & $+$0.137 & $+$0.137 & $+$0.137 $\pm$ 0.000 & deterministic \\
\bottomrule
\end{tabular}

\end{table}

The eligibility rule also matters for interpretation. Table~\ref{tab:pair_count}
and Figure~\ref{fig:pair_count_distribution} show that SIF triples have at
least three source clips by construction, but only a few triples provide more
than three. Figure~\ref{fig:sif_vs_diversity_ratio} checks that the main
pattern is not a monotone artifact of output variation (the diversity ratio $R$).

\begin{table}[t]
  \centering
  \caption{Pair counts per SIF triple. The minimum and median values are fixed
  by the eligibility rule; the few larger cells determine how much within-cell
  geometry is available.}
  \label{tab:pair_count}
  \begin{tabular}{lrrrr}
\toprule
Method & Min pairs & Median pairs & Max pairs & $n$ triples \\
\midrule
ACE-I & 3 & 3 & 5 & 49 \\
ACE-T & 3 & 3 & 5 & 49 \\
AL-Flow & 3 & 3 & 5 & 45 \\
AL-Flow-Src & 3 & 3 & 5 & 38 \\
AL-Flow-Src-G & 3 & 3 & 5 & 49 \\
ANCHOR & 3 & 3 & 5 & 49 \\
AnyTop & 3 & 3 & 5 & 49 \\
DPG-SB-v3 & 3 & 3 & 5 & 49 \\
Motion2Motion-BVH & 3 & 3 & 5 & 49 \\
Motion2Motion-Direct & 3 & 3 & 5 & 49 \\
MoReFlow-I & 3 & 3 & 5 & 49 \\
MoReFlow-T & 3 & 3 & 5 & 49 \\
random-same-exact-action & 3 & 3 & 5 & 48 \\
random-same-cluster & 3 & 3 & 5 & 48 \\
\midrule
\multicolumn{5}{p{0.94\linewidth}}{\footnotesize Min and median are uniformly 3 because the SIF benchmark requires $\geq 3$ source clips per triple by design; max reflects the few triples in Truebones with 4 or 5 source clips. Variation across methods comes from the few queries that some baselines fail to score.} \\
\bottomrule
\end{tabular}

\end{table}

\begin{figure}[t]
  \centering
  \includegraphics[width=0.74\linewidth]{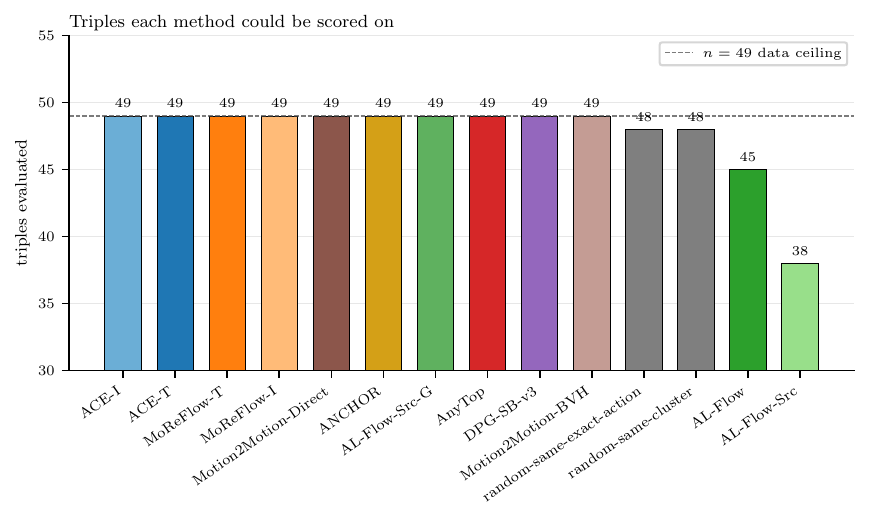}
  \caption{Pair-count distribution for SIF-eligible triples.}
  \label{fig:pair_count_distribution}
\end{figure}

\begin{figure}[t]
  \centering
  \includegraphics[width=0.78\linewidth]{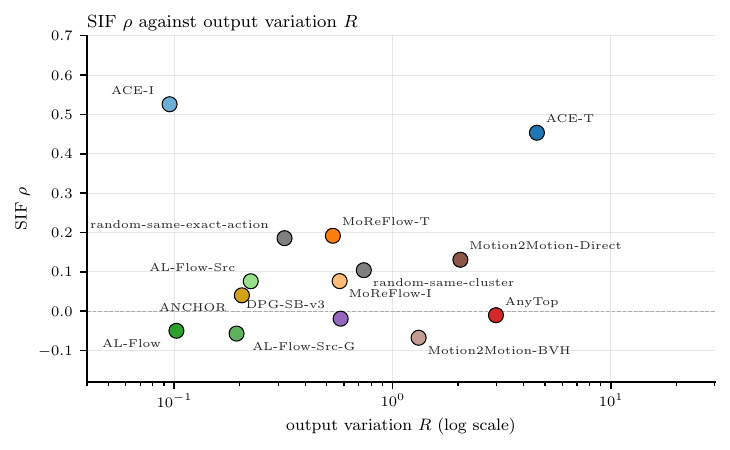}
  \caption{SIF versus output variation (the diversity ratio $R$). High output variation is not
  sufficient for source-instance preservation, because the geometry must align
  with the source-side geometry after the target skeleton and action are fixed.}
  \label{fig:sif_vs_diversity_ratio}
\end{figure}

Latent SIF (L-SIF) repeats the SIF computation before decoding, when a method
exposes a latent representation on the relevant queries. Table~\ref{tab:lsif}
and Figure~\ref{fig:lsif_sif_decoupling} show that latent preservation and
decoded SIF can decouple. This motivates the failure-mode taxonomy used in the
qualitative discussion: a method may collapse in latent space, collapse after
decoding, or keep a positive output correlation while its outputs barely vary
(collapsed variation).

\begin{table}[t]
  \centering
  \caption{SIF and L-SIF for each method. L-SIF is a latent-space
  diagnostic and is therefore reported only where the representation is
  comparable across queries.}
  \label{tab:lsif}
  \begin{tabular}{lrr}
\toprule
Method & SIF $\rho$ & L-SIF $\rho$ \\
\midrule
ACE trained with $z_{\text{src}}=0$ & 0.423 & 0.037 \\
ACE-I (3 seeds avg) & 0.410 & 0.049 \\
ACE-T (3 seeds avg) & 0.408 & 0.088 \\
ACE without adversarial loss (3 seeds avg) & 0.214 & 0.119 \\
MoReFlow-T & 0.192 & -0.200 \\
MoReFlow-I & 0.076 & -0.205 \\
AL-Flow-Src & 0.076 & -0.113 \\
AnyTop & -0.010 & 0.451 \\
AL-Flow & -0.050 & -0.098 \\
AL-Flow-Src-G & -0.057 & -0.126 \\
\midrule
\multicolumn{3}{p{0.94\linewidth}}{\footnotesize Cross-method Pearson correlation between SIF and L-SIF: $r = 0.199$, $p=0.46$ across $n=16$ method runs. The 95\% interval of L-SIF lies above zero for AnyTop and below zero for MoReFlow-I and MoReFlow-T; it includes zero for every other run.} \\
\bottomrule
\end{tabular}

\end{table}

\begin{figure}[t]
  \centering
  \includegraphics[width=0.76\linewidth]{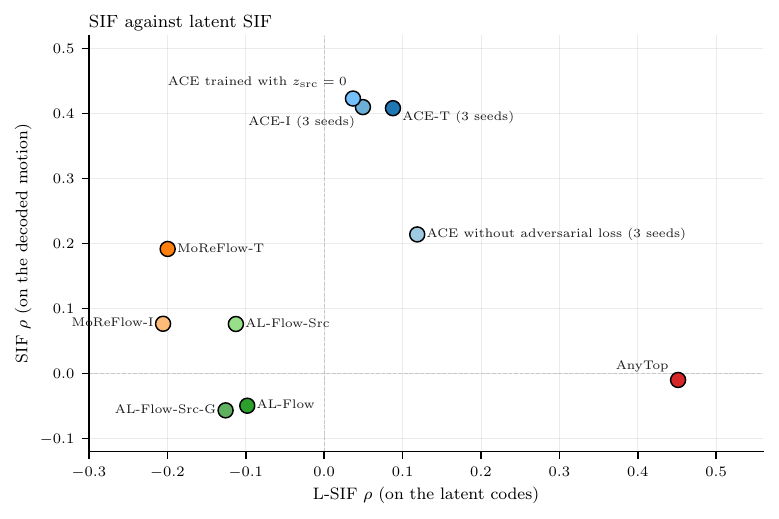}
  \caption{SIF and L-SIF decoupling. Source-instance fidelity in decoded motion
  need not coincide with a positive latent-space correlation.}
  \label{fig:lsif_sif_decoupling}
\end{figure}

\section{Synthetic Calibration}
\label{app:synth}

The synthetic $2\times2$ environment calibrates SIF in a setting where the
source-conditioned transport is known. Figure~\ref{fig:synth_2x2_ladder}
visualizes the three reference regimes used in Section~\ref{sec:sif}: the
oracle transport lies near one, while two source-blind references lie near
zero. Table~\ref{tab:synthetic_sif} records the corresponding numerical
calibration.

\begin{figure}[t]
  \centering
  \includegraphics[width=0.72\linewidth]{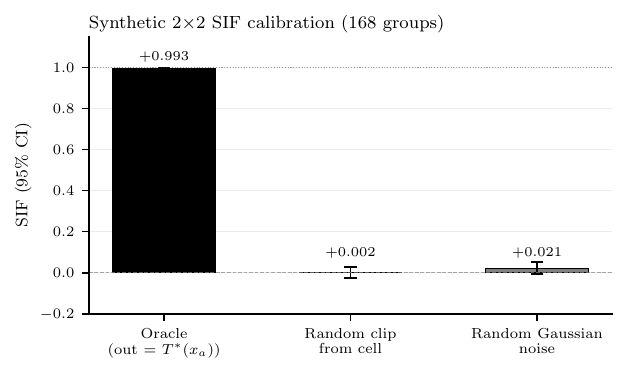}
  \caption{Synthetic $2\times2$ SIF calibration. The oracle uses the known
  source-conditioned transport $T^\star(x_a)$; the random-clip and Gaussian
  references deliberately ignore the source instance.}
  \label{fig:synth_2x2_ladder}
\end{figure}

\begin{table}[t]
  \centering
  \caption{Synthetic SIF calibration. The oracle lies near perfect
  source-instance preservation, while source-blind references remain near zero.}
  \label{tab:synthetic_sif}
  \begin{tabular}{lrrr}
\toprule
Output & SIF (95\% CI) & $p$ & Variation \\
\midrule
Oracle (output $= T^{*}(x_a)$) & \textbf{$+$0.993 [$+$0.992, $+$0.995]} & $<$0.001 & 0.94 \\
Random clip from cell & $+$0.002 [$-$0.027, $+$0.029] & 0.456 & 0.96 \\
Random Gaussian noise & $+$0.021 [$-$0.008, $+$0.054] & 0.164 & $1.3{\times}10^{5}$ \\
\midrule
\multicolumn{4}{p{0.85\linewidth}}{\footnotesize Synthetic 2$\times$2 paired-dense setting, 168 groups of three to six source clips. Another 168 groups are left out because their source clips are identical once rotation and scale are removed, so SIF is undefined there. The oracle saturates SIF; both source-blind references sit at zero. Noise has no motion structure, so its variation is very large.} \\
\bottomrule
\end{tabular}

\end{table}

The same environment also isolates conditional-mean degeneration. The
cell-paired squared-error predictor receives random target samples from the
correct target skeleton and action cell, but not instance-level pairs. As
Table~\ref{tab:propcm_ladder} shows, increasing the number of random
cell-pairs does not recover the oracle within-cell variance.

\begin{table}[t]
  \centering
  \caption{Conditional-mean degeneration in the controlled $2\times2$ setting.
  The predicted variance remains a small fraction of oracle variance even when
  the number of random target-cell pairs increases.}
  \label{tab:propcm_ladder}
  \begin{tabular}{rrrr}
\toprule
$M$ pairs / cell & $\sigma^2_{\text{pred}} / \sigma^2_{\text{oracle}}$ & MSE $/\sigma^2_{\text{oracle}}$ & $n_{\text{seeds}}$ \\
\midrule
2 & 0.020 $\pm$ 0.006 & 1.45 $\pm$ 0.08 & 3 \\
4 & 0.014 $\pm$ 0.002 & 1.29 $\pm$ 0.03 & 3 \\
8 & 0.020 $\pm$ 0.006 & 1.34 $\pm$ 0.08 & 3 \\
16 & 0.015 $\pm$ 0.002 & 1.31 $\pm$ 0.12 & 3 \\
32 & 0.017 $\pm$ 0.001 & 1.22 $\pm$ 0.05 & 3 \\
50 & 0.016 $\pm$ 0.003 & 1.26 $\pm$ 0.09 & 3 \\
\midrule
\multicolumn{4}{p{0.94\linewidth}}{\footnotesize Each source clip is trained against the true target of another clip from the same cell, drawn afresh at every step. Across all $M$ and 3 seeds, the predicted variance stays at about 2\% of the oracle within-cell variance, the cell-mean compression Proposition~\ref{prop:cm} predicts.} \\
\bottomrule
\end{tabular}

\end{table}

\section{Deterministic Comparator and Motion Descriptors}
\label{app:anchor}

ANCHOR is included to make the diagnostic boundary explicit. It is not a
generator and does not synthesize a new target motion. For a query consisting of
a source clip, a target skeleton, and an action request, ANCHOR first predicts a
coarse source action cluster using a random-forest classifier
\citep{Breiman2001RandomForests} trained on motion-space descriptors from the
training skeletons. It then searches the target skeleton's motion library and
returns one existing clip.

The motion descriptor, denoted $Q(x)$, is computed from joint-position
trajectories rather than from a learned latent representation. It summarizes
center-of-mass displacement and speed, heading velocity, contact timing,
cadence, dominant limb usage, and global clip statistics. The candidate pool is
the union of all target-skeleton clips in the predicted cluster and the ten
clips with largest $Q$-descriptor similarity to the source. Candidate clips are
scored by cluster agreement, descriptor similarity, and exact-action agreement,
using the fixed comparator weights
$(w_{\mathrm{cluster}},w_Q,w_{\mathrm{action}})=(1,2,3)$ recorded in the
method roster.

This construction is why ANCHOR is gauge-invariant in the sense of
Corollary~\ref{cor:retr-gauge}: changing a learned latent parametrization
cannot change a score computed directly from center-of-mass, heading, contact,
cadence, and limb-usage measurements on the joint trajectories. At the same
time, because ANCHOR returns an existing library clip, it does not claim to
identify the source-conditioned retargeting map. Its purpose is to mark a
diagnostic boundary, showing that action recovery can remain high at the
source-blind floor.
Table~\ref{tab:anchor_pred} records the action-cluster predictor, and
Figure~\ref{fig:per_cluster_anchor_match} shows how ANCHOR matches vary across
action clusters.

\begin{table}[t]
  \centering
  \caption{ANCHOR action-cluster predictor. The held-out row shows that the
  cluster prediction itself is imperfect, which is why ANCHOR is interpreted as
  a comparator rather than an oracle.}
  \label{tab:anchor_pred}
  \begin{tabular}{lrr}
\toprule
Domain & Cluster classifier accuracy & $n$ correct $/$ total \\
\midrule
In-domain (train skeletons) & 99.8\% & 543 $/$ 544 \\
Held-out (reserved cross-skeleton) & 61.5\% & 59 $/$ 96 \\
\midrule
\multicolumn{3}{p{0.93\linewidth}}{\footnotesize Replacing the ground-truth cluster label with the predicted cluster pulls Procrustes performance below random on exact-tier and on a hand-curated query subset. The in-domain row is the accuracy on the clips the classifier was trained on.} \\
\bottomrule
\end{tabular}

\end{table}

\begin{figure}[t]
  \centering
  \includegraphics[width=0.80\linewidth]{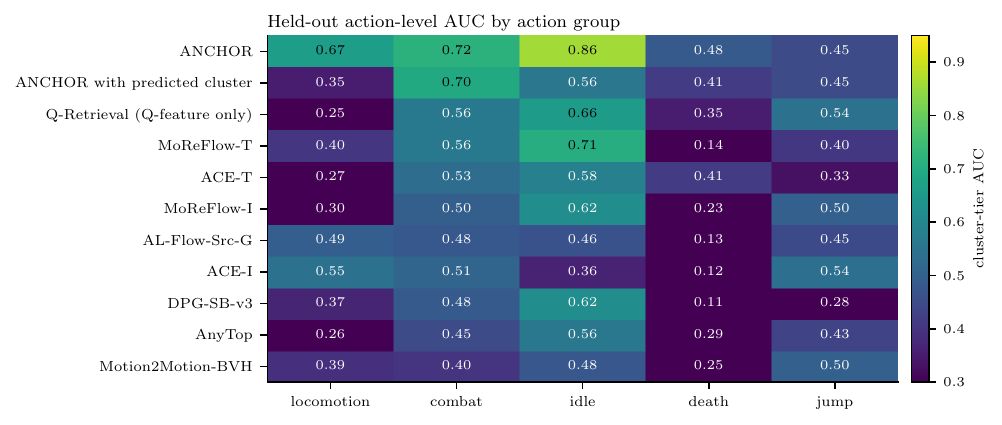}
  \caption{Per-cluster ANCHOR matches. The comparator's behavior varies across
  action clusters, reinforcing that it characterizes a floor rather than a
  universal retargeting rule.}
  \label{fig:per_cluster_anchor_match}
\end{figure}

\section{Action-Level AUC and Label Recovery}
\label{app:enumeration}

The standard cross-skeleton contrastive retrieval AUC is an action-level
ranking test: it asks whether a generated target motion ranks positives above
negatives in a target library. Section~\ref{sec:sif} shows that this
test can accept label-only behavior on the covered queries. The appendix gives
the broader enumeration and comparison tables behind that statement.
Table~\ref{tab:enumeration} and Figure~\ref{fig:enumeration_breakdown} report
the full enumeration, while Table~\ref{tab:gen_variants_summary},
Table~\ref{tab:master_auc}, and Figure~\ref{fig:master_auc_band} collect the
AUC-band controls.

\begin{table}[t]
  \centering
  \caption{Full ANCHOR enumeration across cluster-eligible triples. The held-out
  subsets show that strong action-level ranking can arise from target-library
  label structure without identifying a source-conditioned map.}
  \label{tab:enumeration}
  \begin{tabular}{lrr}
\toprule
Subset & $n$ triples (cluster-eligible) & Cluster-tier AUC (95\% CI) \\
\midrule
All & 30{,}029 & 0.824 [0.820, 0.829] \\
In-distribution (same train pair) & 21{,}838 & 0.845 [0.841, 0.849] \\
Mixed (one held-out skeleton) & 7{,}622 & \textbf{0.774 [0.762, 0.787]} \\
Held-out (both cross-skeleton) & 569 & 0.711 [0.669, 0.758] \\
\midrule
\multicolumn{3}{p{0.96\linewidth}}{\footnotesize ANCHOR on the 30{,}497-pair full enumeration. Cluster-eligibility filter removes triples whose target skeleton has no in-cluster positives. The mixed subset is the held-out case.} \\
\bottomrule
\end{tabular}

\end{table}

\begin{figure}[t]
  \centering
  \includegraphics[width=0.78\linewidth]{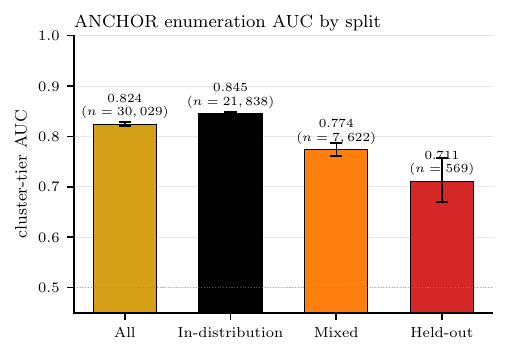}
  \caption{Enumeration breakdown for ANCHOR. The large enumeration separates
  in-distribution, mixed, and held-out skeleton regimes.}
  \label{fig:enumeration_breakdown}
\end{figure}

\begin{table}[t]
  \centering
  \caption{Action-level AUC for generated variants and supervision-matched
  comparators. These values are not interpreted as source-preserving transfer;
  they show how much action-level evidence can be explained without SIF.}
  \label{tab:gen_variants_summary}
  \resizebox{\linewidth}{!}{{\small
\begin{tabular}{p{0.22\linewidth}p{0.50\linewidth}r}
\toprule
Variant & Conditioning channels & Cluster AUC \\
\midrule
AnyTop & self-supervised source-conditioned diffusion & 0.465 \\
Precursor action-conditioned generator & action-conditioned generation & 0.485 \\
Precursor latent-bridge generator & latent bridge objective & 0.483 \\
DPG-SB-v3 & latent bridge objective; no decoded-motion penalties & 0.447 \\
AL-Flow & cluster + exact-action labels & 0.538 \\
AL-Flow-Src & AL-Flow + source motion + skeleton identity & 0.542 \\
AL-Flow-Src-G & AL-Flow + source motion + skeleton graph & 0.470 \\
\midrule
\multicolumn{3}{p{0.94\linewidth}}{\footnotesize The two precursor rows are reported only for the action-level AUC analysis and are not part of the 14-method SIF roster. AL-Flow and AL-Flow-Src produce outputs for 159 and 100 of the 300 queries in each fold, and their AUC is computed on those queries.} \\
\bottomrule
\end{tabular}
}
}
\end{table}

\begin{table}[t]
  \centering
  \caption{Master action-level AUC comparison on held-out queries. The table
  reports cluster-tier AUC under three distances: Procrustes trajectory
  distance; Z-DTW, dynamic time warping on z-normalized body-part trajectories;
  and Q-comp, a combination of centre-of-mass path, foot-contact timing,
  cadence, and limb-usage differences. Library-based references are listed
  separately from generative methods.}
  \label{tab:master_auc}
  \begin{tabular}{lrrr}
\toprule
\multicolumn{4}{c}{Cluster-tier held-out AUC (average of folds 42 and 43)} \\
Method & Procrustes & Z-DTW & Q-comp \\
\midrule
Action oracle & 0.873 & 0.846 & 0.894 \\
Self-positive reference & 0.860 & 0.847 & 0.909 \\
\midrule
ANCHOR & 0.681 & 0.704 & 0.728 \\
Cluster-Classifier Retrieval & 0.646 & 0.669 & 0.693 \\
Q-Retrieval (Q-feature only) & 0.453 & 0.523 & 0.597 \\
\midrule
MoReFlow-T & 0.463 & 0.496 & 0.558 \\
MoReFlow-I & 0.419 & 0.431 & 0.567 \\
ACE-T & 0.448 & 0.451 & 0.529 \\
ACE-I & 0.401 & 0.469 & 0.550 \\
\midrule
Random target skeleton (null) & 0.474 & 0.510 & 0.530 \\
\midrule
\multicolumn{4}{l}{\footnotesize Motion2Motion-BVH has Procrustes only and is reported separately.} \\
\bottomrule
\end{tabular}

\end{table}

\begin{figure}[t]
  \centering
  \includegraphics[width=0.78\linewidth]{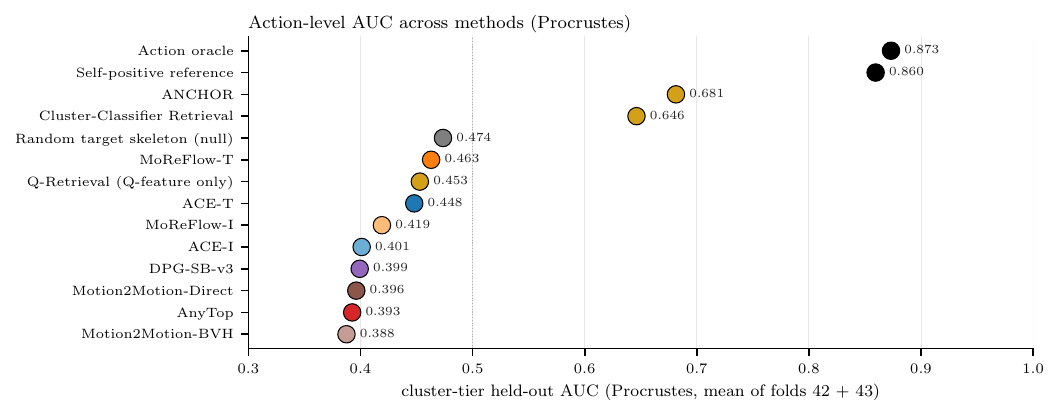}
  \caption{Action-level AUC band. Many methods occupy a narrow action-ranking
  band, which motivates SIF as the missing source-instance axis.}
  \label{fig:master_auc_band}
\end{figure}

\section{ACE Controls}
\label{app:ace_controls}

ACE is the evaluated method with the largest raw SIF, so we traced where its
signal comes from. This appendix removes or shuffles its source input, removes
its adversarial loss in a matched comparison, and checks how its output
variation should be measured.

ACE receives the source clip through a source latent, $z_{\text{src}}$, and in
our evaluation each output also took its length from its source clip.
Table~\ref{tab:ace_source_removal} varies both inputs at generation on the 49
original triples, for ACE-T and for an ablation trained with $z_{\text{src}}$ set
to zero. That ablation still received the real source latent at generation and
copied the source length, so its raw SIF of about $+0.42$ comes from these two
inputs. When the source latent is removed at generation and the length is
fixed, both models score exactly zero and their outputs no longer vary. This is
the reference point for the source-blind floor. With the source removed but the
length copied, both still score about $+0.30$, the share carried by length
alone. Shuffling the source among the clips of a triple brings the score to zero
or below. With the real source and a fixed length, ACE-T keeps $+0.29$, which is
$+0.29$ above the removed-source result (95\% CI $[+0.12,+0.47]$). Three
training seeds give $+0.29 \pm 0.07$.

\begin{table}[t]
  \centering
  \caption{Removing or shuffling ACE's source input at generation (49 original
  triples). Removed means $z_{\text{src}}$ set to zero at generation; shuffled
  means each output receives the source of another clip in the same triple;
  fixed means every output is generated with 64 frames. $p$ is the
  shuffle-test $p$-value; variation as in Table~\ref{tab:sif_main}.}
  \label{tab:ace_source_removal}
  \resizebox{\linewidth}{!}{\begin{tabular}{llrrrrrr}
\toprule
 & & \multicolumn{3}{c}{ACE-T} & \multicolumn{3}{c}{Model trained with $z_{\text{src}}=0$} \\
\cmidrule(lr){3-5}\cmidrule(lr){6-8}
Source at generation & Output length & SIF & $p$ & Variation & SIF & $p$ & Variation \\
\midrule
real & copied & $+$0.385 & $<$0.001 & 0.020 & $+$0.419 & $<$0.001 & 0.023 \\
real & fixed & $+$0.294 & $<$0.001 & 0.003 & $+$0.213 & 0.006 & 0.004 \\
removed & copied & $+$0.300 & $<$0.001 & 0.009 & $+$0.299 & $<$0.001 & 0.007 \\
removed & fixed & 0.000 & 1.000 & 0.000 & 0.000 & 1.000 & 0.000 \\
shuffled & copied & $-$0.192 & 0.987 & 0.020 & $-$0.165 & 0.970 & 0.023 \\
shuffled & fixed & $-$0.064 & 0.774 & 0.003 & $+$0.040 & 0.328 & 0.004 \\
\bottomrule
\end{tabular}
}
\end{table}

To test whether the adversarial loss creates this signal, we compare ACE-I with
a model trained without that loss but otherwise identically: the same 60
training skeletons, the same data, and the same three seeds.
Table~\ref{tab:ace_adv_matched} reports both scorings. When the output length
follows the source, the model with the adversarial loss scores higher. When the
length is fixed at generation, the difference disappears on both evaluation
sets. The difference therefore depends on the output length following the
source. Without the adversarial loss, the outputs also vary more. A further
ablation without the source-feature loss diverged during training and gave no
usable model.

\begin{table}[t]
  \centering
  \caption{ACE-I with and without the adversarial loss, trained identically
  otherwise (three seeds each). SIF is averaged over seeds; variation gives the
  range of the per-seed medians; the difference is paired over triples, with a
  95\% interval that resamples whole source skeletons.}
  \label{tab:ace_adv_matched}
  \resizebox{\linewidth}{!}{\begin{tabular}{lrrrrrlr}
\toprule
 & & \multicolumn{2}{c}{With adversarial loss} & \multicolumn{2}{c}{Without} & & \\
\cmidrule(lr){3-4}\cmidrule(lr){5-6}
Scoring & Triples & SIF & Variation & SIF & Variation & Difference [95\% CI] & $p$ \\
\midrule
raw, length copied & 49 & $+$0.407 & 0.017--0.022 & $+$0.228 & 0.271--0.906 & $+$0.179 [$+$0.04, $+$0.33] & 0.022 \\
length fixed at generation & 49 & $+$0.292 & 0.002--0.004 & $+$0.266 & 0.073--0.107 & $+$0.026 [$-$0.09, $+$0.15] & 0.677 \\
length fixed at generation & 1{,}891 & $+$0.329 & 0.002--0.003 & $+$0.329 & 0.078--0.111 & 0.000 [$-$0.05, $+$0.06] & 0.999 \\
\bottomrule
\end{tabular}
}
\end{table}

Output variation needs care under length control. Resampling finished outputs
to 64 frames stretches clips of different lengths by different amounts, which
can make identical motions look different. Table~\ref{tab:ace_stretch} shows
that ACE's outputs, nearly identical in raw scoring, gain a variation of 0.14 to
0.23 after stretching. Identical copies of a single ACE output, cut to the
lengths of its siblings, reach almost the same value. MoReFlow's variation is
the same with and without stretching and clearly exceeds what stretching alone
produces. We therefore report ACE's variation in raw scoring and with the length
fixed at generation.

\begin{table}[t]
  \centering
  \caption{Output variation after stretching every output to 64 frames (49
  original triples). Identical copies are one output cut to the lengths of its
  siblings, so any variation they show comes from stretching alone.}
  \label{tab:ace_stretch}
  \small
  \begin{tabular}{lrrr}
\toprule
 & \multicolumn{2}{c}{Actual outputs} & Identical copies, \\
\cmidrule(lr){2-3}
Method & Raw & Stretched to 64 frames & stretched to 64 frames \\
\midrule
ACE-T & 0.020 & 0.231 & 0.149 \\
ACE-I & 0.020 & 0.135 & 0.142 \\
MoReFlow-T & 0.247 & 0.217 & 0.093 \\
MoReFlow-I & 0.276 & 0.279 & 0.078 \\
\bottomrule
\end{tabular}

\end{table}

The nearest-neighbor leakage check in Table~\ref{tab:ace_leakage_nn} addresses a
different explanation: that ACE could be lifting SIF by memorizing source clips.
Most ACE outputs are closer to a target-pool training clip than to the
corresponding source clip in the reported four-dimensional kinematic feature
space, which weakens source memorization as the explanation for the SIF lift.

\begin{table}[t]
  \centering
  \caption{ACE nearest-neighbor leakage check. Distances are computed in a
  kinematic feature space; the majority of ACE outputs are closer to the target
  pool than to the source clip.}
  \label{tab:ace_leakage_nn}
  \begin{tabular}{lrlrrr}
\toprule
Method & Queries & \begin{tabular}[b]{@{}r@{}}Closer to the\\target pool\end{tabular} & \begin{tabular}[b]{@{}r@{}}Mean distance\\to source\end{tabular} & \begin{tabular}[b]{@{}r@{}}Mean distance\\to target pool\end{tabular} & Ratio \\
\midrule
ACE-I & 130 & 103 (79.2\%) & 3.449 & 2.215 & 1.56 \\
ACE-T & 130 & 99 (76.2\%) & 3.697 & 2.650 & 1.39 \\
\midrule
\multicolumn{6}{p{0.94\linewidth}}{\footnotesize Distance computed in a four-dimensional kinematic feature space (centre-of-mass displacement, centre-of-mass variance, mean velocity, foot-contact density). 76--79\% of ACE outputs are closer to a target-pool training clip than to the corresponding source clip, weakening source-memorisation as an explanation of ACE's SIF lift.} \\
\bottomrule
\end{tabular}

\end{table}

\section{Human-to-Robot Studies}
\label{app:robots}

Truebones has no true correspondences, so it can show that methods sit at the
source-blind floor but not what a correct map would score. The two studies in
this appendix supply true pairs. They test SIF where the answer is known, and
they train the two objective classes of Section~\ref{sec:structural} on the same
data as a model given the true pairs.

\paragraph{Data.}
The first study uses BONES-SEED~\citep{bonesseed}, a motion-capture dataset whose
human clips come with retargeted counterparts on the Unitree G1 humanoid (Motion
Data by Bones Studio, \url{https://bones.studio/}). We selected 90 action groups
with three or four clips each, 355 clips in total. The second study starts from
the 77 recordings of LAFAN1~\citep{harvey2020robust}. We cut each recording into
10-second clips and grouped the clips by routine; each routine is performed by
up to five people, giving 29 groups. We converted every clip to six humanoid
robots with GMR~\citep{araujo2025gmr}, a public inverse-kinematics retargeting
tool: Unitree G1, Booster T1, Fourier N1, Stanford Toddy, EngineAI PM01, and PAL
Talos, which differ in height and have 23 to 31 joints. The converted clip serves
as the true counterpart. LAFAN1 is released under CC BY-NC-ND 4.0; the paper
shows rendered stills for illustration and does not release the retargeted
clips. GMR is released under the MIT License, and the robot models are those
distributed with GMR, credited to their original sources in its repository. We
use two settings: three clips per group, as in
Truebones (72 clips; 17 groups have at least three), and a denser one (154
clips; 26 groups). All six-robot outputs are scored on a common 256-frame window.
In both studies the sparsity is created by us: the pairs exist, but the two
theorem objectives never see them.

\paragraph{Models.}
Figure~\ref{fig:robot_renders} shows the data and the trained models on the
robots. In both studies we train one small transformer (four layers, width 256) in three
ways, with the same data, number of steps, and optimizer. The unpaired objective
couples two variational autoencoders, one per body, through a shared Gaussian
latent and trains each only on its own motions, which places it in the class of
Theorem~\ref{thm:gauge}. The averaging objective regresses each human clip, with
squared error, onto a clip of the same group on the robot, drawn at random at
every step, which places it in the class of Proposition~\ref{prop:cm}. The
true-pair model uses the same regressor trained on the true counterpart. Two
reference rows need no training: the true counterpart itself and a random clip
from the same group.

\begin{figure}[t]
  \centering
  \includegraphics[width=\linewidth]{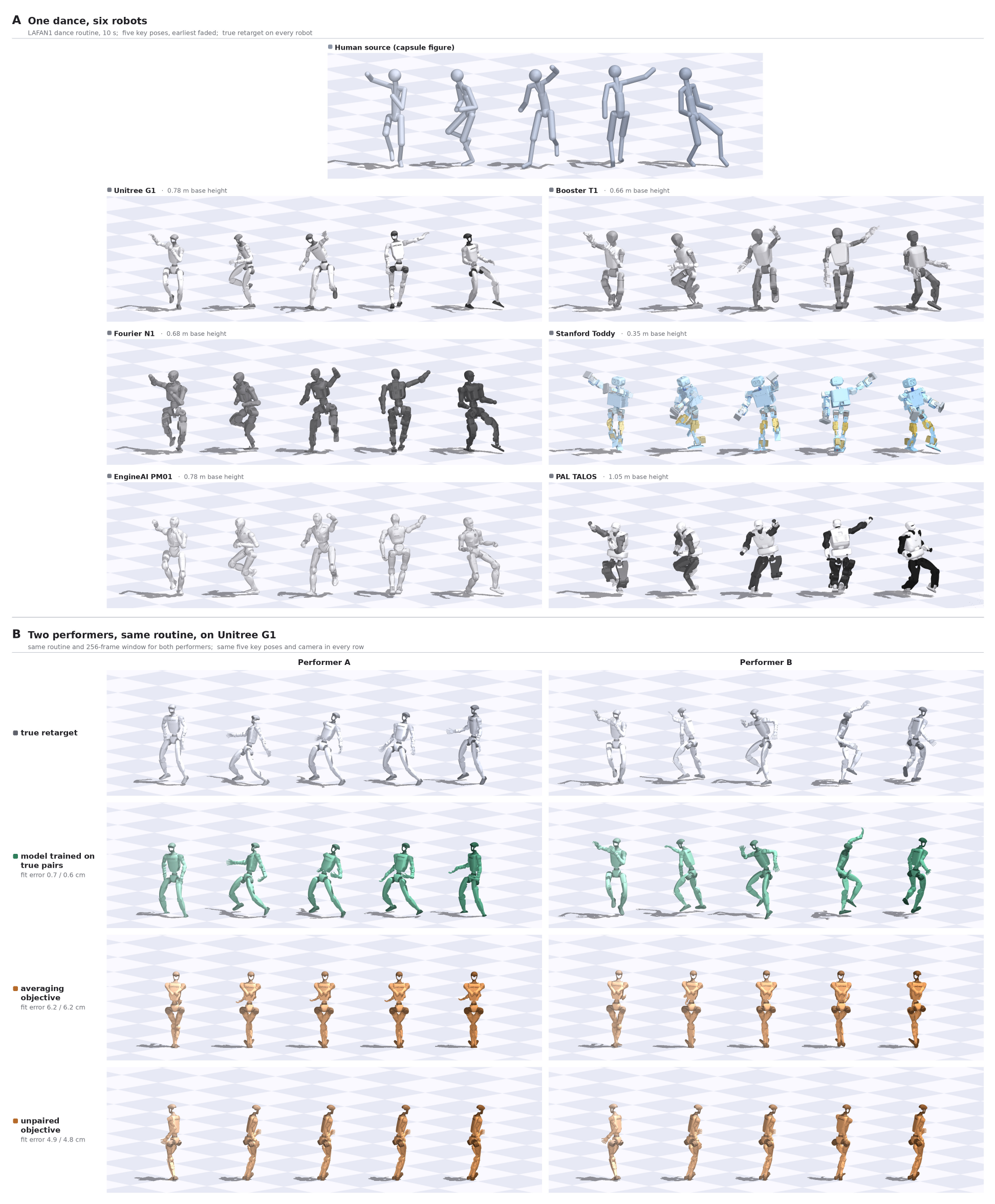}
  \caption{Human-to-robot renders. (A) One LAFAN1 dance routine (10\,s)
  retargeted by GMR to six humanoid robots; five key poses per strip, earliest
  faded. Each strip is framed to its robot, so compare poses, not sizes; the base
  heights give the true scale. (B) Two performers dancing the same routine, and
  what each model produces for each of them on the Unitree G1. The true
  retargets and the model trained on true pairs differ between the performers.
  The averaging and unpaired objectives produce nearly the same motion for both,
  and that motion barely changes over time. Model outputs are predicted body
  positions, drawn as the nearest G1 pose; the fit error is listed under each
  row. The 5 to 6\,cm errors of the two collapsed rows mean their raw
  predictions are less robot-like than drawn. The two performers are not aligned
  in time, so part of the difference in the top rows is timing rather than
  style.}
  \label{fig:robot_renders}
\end{figure}

\paragraph{Human-to-G1 results.}
Table~\ref{tab:robot_g1_full} extends Table~\ref{tab:g1_joint} with
length-controlled SIF. The true counterpart scores the same with and without
length control (0.960 and 0.959), so SIF on real pairs does not depend on clip
length. When scoring action AUC, each output's own clip is removed from the reference
set, as the query's clip is for every output.

\begin{table}[t]
  \centering
  \caption{Human-to-G1 study, all four measures on identical test items (90
  groups, 355 clips), with SIF under both scorings. Variation is the median ratio
  of output spread to source spread; realistic as in Table~\ref{tab:g1_joint}.}
  \label{tab:robot_g1_full}
  \small
  \begin{tabular}{lrrrrr}
\toprule
 & Raw SIF & Length-controlled SIF & Action AUC & Variation & Realistic \\
\midrule
True retargeted counterpart & $+$0.960 & $+$0.959 & 0.980 & 0.518 & 97.2\% \\
Random same-action clip & $-$0.008 & $-$0.068 & 0.978 & 0.414 & 97.5\% \\
Unpaired objective & $+$0.203 & $+$0.154 & 0.499 & 0.001 & 100.0\% \\
Averaging objective & $+$0.379 & $+$0.343 & 0.820 & 0.048 & 53.0\% \\
Model trained on true pairs & $+$0.900 & $+$0.698 & 0.967 & 0.521 & 94.6\% \\
\bottomrule
\end{tabular}

\end{table}

\paragraph{Six-robot predictions.}
For the six-robot study we wrote down five predictions with numeric thresholds
before scoring any trained model; we had seen the reference rows for one robot
during a pilot run. Here $Q$ is a model's output variation divided by that of the
true counterpart, matched group by group. Each prediction is checked 24 times:
six robots, two settings, and raw or length-controlled scoring.
Table~\ref{tab:robot_predictions} reports the outcome, including the misses, and
Table~\ref{tab:robot_per_robot} gives every robot. The true counterpart falls
just short of the 0.80 threshold on Booster T1 (0.79) and in one Stanford Toddy
evaluation (0.77), but it is clearly above zero everywhere. The random-clip
prediction holds in every denser evaluation; with three clips per group its
intervals are too wide to decide. The unpaired objective keeps about 40 percent
of the variation on PAL Talos with three clips per group, and collapses in every
other evaluation.

\begin{table}[t]
  \centering
  \caption{Six-robot predictions, fixed before any trained model was scored, and
  their outcomes over 24 evaluations (six robots, two settings, two scorings).
  $Q$ is a model's output variation divided by that of the true counterpart,
  matched group by group; intervals are group bootstraps.}
  \label{tab:robot_predictions}
  \small
  \begin{tabular}{p{0.28\linewidth}p{0.26\linewidth}p{0.34\linewidth}}
\toprule
Prediction & Threshold & Result \\
\midrule
The true retarget is recognized & SIF $\geq 0.80$, interval above 0 & 19 of 24; every interval above 0; misses 0.768--0.799 \\
A random same-group clip sits at the floor & 90\% interval within $\pm 0.20$ & 12 of 24; 12 of 12 in the denser setting \\
The averaging objective loses variation & variation $Q \leq 0.35$, upper bound $< 0.50$ & 24 of 24 \\
The unpaired objective loses variation & variation $Q \leq 0.35$, upper bound $< 0.50$ & 22 of 24; missed with variation $Q = 0.41$ \\
The true-pair model recovers both & SIF $\geq 0.60$ and variation $Q \geq 0.70$ & 24 of 24; SIF 0.77--0.99 \\
\bottomrule
\end{tabular}

\end{table}

\begin{table}[t]
  \centering
  \caption{Six-robot results for every robot, setting, and scoring. SIF columns
  give the correlation; $Q$ columns give output variation relative to the true
  counterpart.}
  \label{tab:robot_per_robot}
  \resizebox{\linewidth}{!}{\begin{tabular}{lllrrrrrr}
\toprule
 & & & True retarget & Random clip & Unpaired & Averaging & \multicolumn{2}{c}{True-pair model} \\
\cmidrule(lr){8-9}
Robot & Setting & Scoring & SIF & SIF & Variation ($Q$) & Variation ($Q$) & SIF & Variation ($Q$) \\
\midrule
Unitree G1 & three-clip & raw & $+$0.988 & $+$0.066 & 0.059 & 0.286 & $+$0.987 & 1.00 \\
Unitree G1 & three-clip & length-controlled & $+$0.987 & $+$0.066 & 0.059 & 0.286 & $+$0.987 & 1.00 \\
\addlinespace
Unitree G1 & denser & raw & $+$0.966 & $-$0.061 & 0.012 & 0.035 & $+$0.965 & 0.99 \\
Unitree G1 & denser & length-controlled & $+$0.966 & $-$0.058 & 0.012 & 0.035 & $+$0.965 & 0.99 \\
\addlinespace
Booster T1 & three-clip & raw & $+$0.799 & $+$0.024 & 0.055 & 0.347 & $+$0.845 & 0.99 \\
Booster T1 & three-clip & length-controlled & $+$0.788 & $+$0.026 & 0.057 & 0.346 & $+$0.842 & 0.99 \\
\addlinespace
Booster T1 & denser & raw & $+$0.790 & $-$0.062 & 0.036 & 0.067 & $+$0.772 & 0.99 \\
Booster T1 & denser & length-controlled & $+$0.790 & $-$0.059 & 0.036 & 0.067 & $+$0.772 & 0.99 \\
\addlinespace
Fourier N1 & three-clip & raw & $+$0.883 & $+$0.098 & 0.027 & 0.278 & $+$0.882 & 1.00 \\
Fourier N1 & three-clip & length-controlled & $+$0.883 & $+$0.099 & 0.026 & 0.280 & $+$0.881 & 1.01 \\
\addlinespace
Fourier N1 & denser & raw & $+$0.899 & $-$0.047 & 0.019 & 0.044 & $+$0.895 & 1.00 \\
Fourier N1 & denser & length-controlled & $+$0.898 & $-$0.044 & 0.019 & 0.044 & $+$0.894 & 1.00 \\
\addlinespace
Stanford Toddy & three-clip & raw & $+$0.768 & $+$0.182 & 0.017 & 0.267 & $+$0.866 & 1.00 \\
Stanford Toddy & three-clip & length-controlled & $+$0.877 & $+$0.185 & 0.017 & 0.267 & $+$0.872 & 1.00 \\
\addlinespace
Stanford Toddy & denser & raw & $+$0.861 & $-$0.050 & 0.190 & 0.049 & $+$0.862 & 0.99 \\
Stanford Toddy & denser & length-controlled & $+$0.860 & $-$0.047 & 0.191 & 0.049 & $+$0.861 & 0.99 \\
\addlinespace
EngineAI PM01 & three-clip & raw & $+$0.887 & $+$0.099 & 0.316 & 0.294 & $+$0.885 & 1.00 \\
EngineAI PM01 & three-clip & length-controlled & $+$0.885 & $+$0.099 & 0.317 & 0.295 & $+$0.883 & 1.00 \\
\addlinespace
EngineAI PM01 & denser & raw & $+$0.906 & $-$0.051 & 0.027 & 0.051 & $+$0.907 & 0.99 \\
EngineAI PM01 & denser & length-controlled & $+$0.907 & $-$0.048 & 0.027 & 0.051 & $+$0.908 & 0.99 \\
\addlinespace
PAL Talos & three-clip & raw & $+$0.895 & $+$0.027 & 0.405 & 0.260 & $+$0.892 & 1.00 \\
PAL Talos & three-clip & length-controlled & $+$0.896 & $+$0.026 & 0.405 & 0.261 & $+$0.892 & 1.00 \\
\addlinespace
PAL Talos & denser & raw & $+$0.878 & $-$0.043 & 0.134 & 0.051 & $+$0.878 & 0.99 \\
PAL Talos & denser & length-controlled & $+$0.878 & $-$0.041 & 0.134 & 0.051 & $+$0.878 & 0.99 \\
\bottomrule
\end{tabular}
}
\end{table}

\paragraph{An adversarial objective without pairs.}
We also trained an objective in the style of ACE on the human-to-G1 data: a
generator with a motion-space discriminator and a source-feature loss, without
pairs. The planned run diverged: its outputs grew without bound, and none passed
the realism check. We then re-ran it post hoc with ACE's own optimizer settings,
which stabilized training. Table~\ref{tab:robot_adversarial} reports all three
seeds. No seed scores well on correlation, action accuracy, variation, and
realism together.

\begin{table}[t]
  \centering
  \caption{An adversarial objective in the style of ACE on the human-to-G1 data,
  without pairs. The planned run diverged; the re-runs use ACE's optimizer
  settings and are post hoc.}
  \label{tab:robot_adversarial}
  \small
  \begin{tabular}{lrrrr}
\toprule
 & Raw SIF & Action AUC & Variation & Realistic \\
\midrule
Planned run & $+$0.209 & 0.500 & diverged & 0.0\% \\
Re-run, seed 1 & $+$0.696 & 0.850 & 0.561 & 14.4\% \\
Re-run, seed 2 & $+$0.528 & 0.551 & 0.093 & 78.0\% \\
Re-run, seed 3 & $+$0.574 & 0.513 & 0.101 & 0.0\% \\
\bottomrule
\end{tabular}

\end{table}

\paragraph{Four measures on identical test items.}
Table~\ref{tab:truebones_fourmeasure} reports SIF, action AUC, output variation,
and realism together for the Truebones methods on the 49 original triples. No
method is high on both SIF and action AUC: the retrieval and label-only rows
reach the highest action AUC with SIF near zero, and ACE reaches the highest SIF
with chance-level action AUC.

\begin{table}[t]
  \centering
  \caption{Truebones methods on the 49 original triples: SIF (raw), action AUC,
  output variation, and realism on identical test items. Realism is undefined for
  Motion2Motion-BVH, whose outputs use a different joint layout.}
  \label{tab:truebones_fourmeasure}
  \small
  \begin{tabular}{lrrrr}
\toprule
Method & Raw SIF & Action AUC & Variation & Realistic \\
\midrule
ACE-I & $+$0.484 & 0.539 & 0.02 & 91.8\% \\
ACE-T & $+$0.385 & 0.538 & 0.02 & 92.4\% \\
MoReFlow-T & $+$0.203 & 0.596 & 0.25 & 99.4\% \\
MoReFlow-I & $+$0.137 & 0.578 & 0.28 & 99.4\% \\
random-same-exact-action & $+$0.096 & 0.799 & 0.00 & 100.0\% \\
random-same-cluster & $+$0.073 & 0.813 & 0.25 & 100.0\% \\
AL-Flow-Src & $+$0.070 & 0.633 & 0.11 & 100.0\% \\
Motion2Motion-Direct & $+$0.053 & 0.514 & 1.69 & 58.5\% \\
ANCHOR & $+$0.014 & 0.828 & 0.00 & 100.0\% \\
AL-Flow-Src-G & $-$0.028 & 0.606 & 0.08 & 100.0\% \\
AnyTop & $-$0.033 & 0.518 & 1.34 & 52.0\% \\
DPG-SB-v3 & $-$0.058 & 0.684 & 0.22 & 100.0\% \\
AL-Flow & $-$0.066 & 0.643 & 0.04 & 100.0\% \\
Motion2Motion-BVH & $-$0.085 & 0.449 & 0.93 & -- \\
\bottomrule
\end{tabular}

\end{table}

Table~\ref{tab:robot6_fourmeasure} gives the same measures for the six-robot
study. The model trained on true pairs matches the true counterpart on all four.
The averaging objective keeps little of the variation and produces many
unrealistic poses, and the unpaired objective falls to chance on action AUC.
Action AUC is lower for every case here, including the true counterpart, because
several LAFAN1 categories are all forms of locomotion. Action AUC and realism for
this study were computed after the pre-registration, with the same definitions as
the human-to-G1 study.

\begin{table}[t]
  \centering
  \caption{Six-robot study, four measures on identical test items: mean over the
  six robots, with the range in parentheses. Variation is $Q$, a model's output
  variation divided by that of the true counterpart. Action AUC and realism were
  computed after the pre-registration.}
  \label{tab:robot6_fourmeasure}
  \resizebox{\linewidth}{!}{\begin{tabular}{lllll}
\toprule
 & Raw SIF & Action AUC & Variation ($Q$) & Realistic \\
\midrule
\multicolumn{5}{l}{\emph{Three clips per group}} \\
True retargeted counterpart & +0.87 (+0.77--+0.99) & 0.61 (0.60--0.63) & 1.00 (1.00--1.00) & 99 (99--99)\% \\
Random same-group clip & +0.08 (+0.02--+0.18) & 0.62 (0.61--0.63) & 0.73 (0.69--0.77) & 99 (99--100)\% \\
Unpaired objective & +0.24 (+0.11--+0.36) & 0.50 (0.49--0.52) & 0.15 (0.02--0.41) & 98 (89--100)\% \\
Averaging objective & +0.11 (+0.07--+0.20) & 0.60 (0.58--0.63) & 0.29 (0.26--0.35) & 46 (33--83)\% \\
Model trained on true pairs & +0.89 (+0.85--+0.99) & 0.61 (0.60--0.63) & 1.00 (0.99--1.00) & 98 (97--99)\% \\
\multicolumn{5}{l}{\emph{Denser setting}} \\
True retargeted counterpart & +0.88 (+0.79--+0.97) & 0.60 (0.58--0.61) & 1.00 (1.00--1.00) & 99 (99--99)\% \\
Random same-group clip & $-$0.05 ($-$0.06--$-$0.04) & 0.59 (0.57--0.61) & 0.82 (0.79--0.85) & 98 (98--99)\% \\
Unpaired objective & +0.15 (+0.10--+0.20) & 0.51 (0.50--0.52) & 0.07 (0.01--0.19) & 61 (1--100)\% \\
Averaging objective & +0.16 (+0.12--+0.18) & 0.56 (0.55--0.59) & 0.05 (0.03--0.07) & 23 (8--84)\% \\
Model trained on true pairs & +0.88 (+0.77--+0.97) & 0.60 (0.58--0.61) & 0.99 (0.99--1.00) & 99 (99--99)\% \\
\bottomrule
\end{tabular}
}
\end{table}

\section{Qualitative Evidence and Supplementary Videos}
\label{app:videos}

The qualitative figures are not used as quantitative proof, but they make the
SIF failure modes visible. The appendix provides three full-grid still figures
that correspond to the supplementary videos. Each grid fixes a source-target
triple and compares the evaluated methods across three source clips, so rows
that move in near lockstep indicate source-blind behavior while rows with
source-dependent variation indicate partial source conditioning.
Figure~\ref{fig:qual_method_panorama} gives the compact method panorama in the
main text. Figures~\ref{fig:qualitative_failuremodes},
\ref{fig:qual_full_a}, \ref{fig:qual_full_b}, and \ref{fig:qual_full_c} provide
the appendix still-image views corresponding to this qualitative protocol.

\begin{figure}[t]
  \centering
  \includegraphics[width=0.96\linewidth]{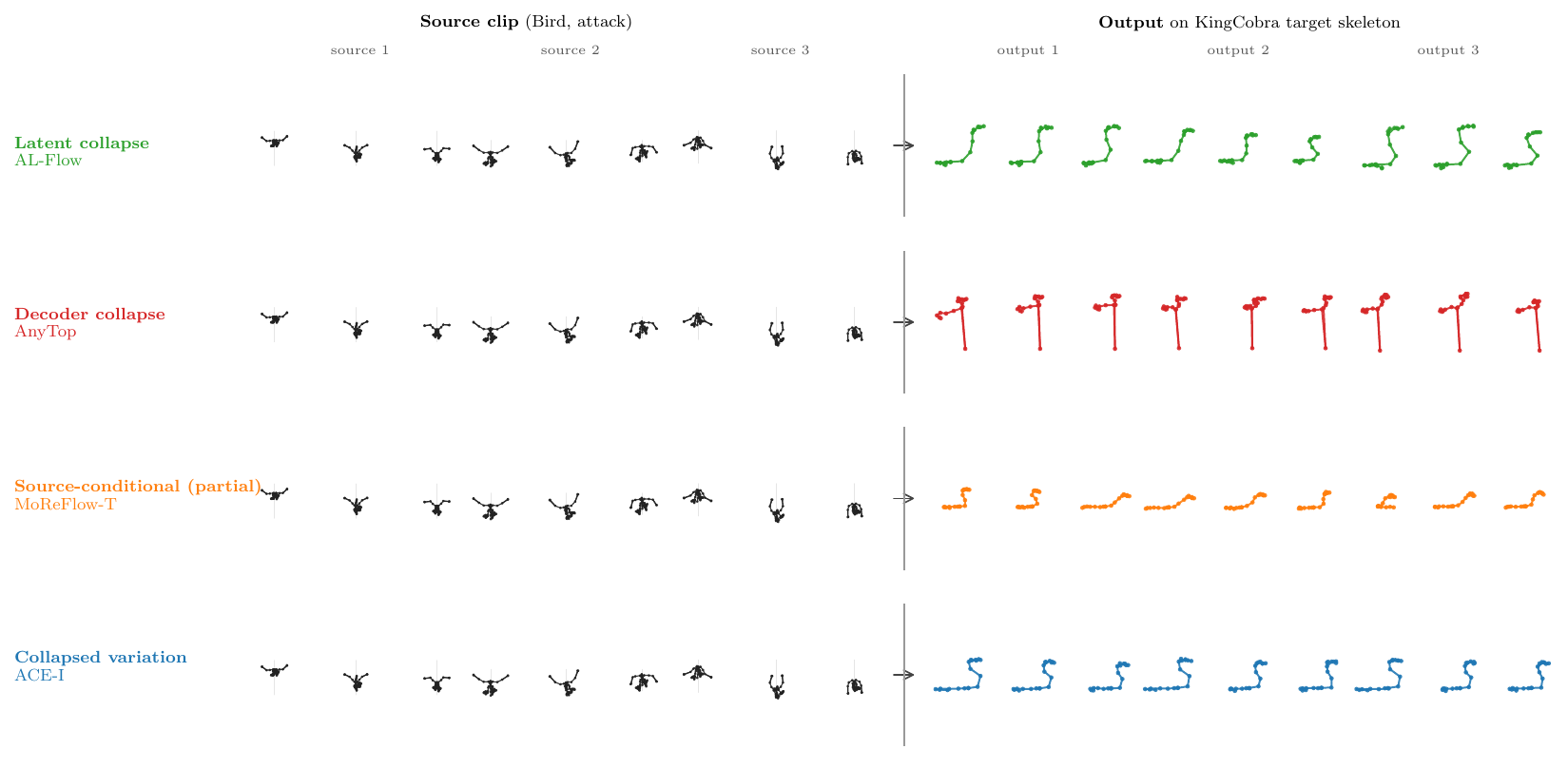}
  \caption{A source-clip substitution view on one Bird $\to$ KingCobra attack
  triple. The first three columns, labeled source 1 to source 3, are different
  Bird attack clips, while the target skeleton and action are fixed; the three
  columns on the right, labeled output 1 to output 3, show the corresponding
  KingCobra motions generated by each row. Each row is labeled by its
  source-instance failure mode: latent collapse for AL-Flow, where source
  identity is lost before decoding; decoder collapse for AnyTop, where latent
  variation is not expressed in the output; partial source-conditionality for
  MoReFlow-T, the clearest output-level source dependence among the displayed
  rows; and collapsed variation for ACE-I, whose outputs keep a positive
  correlation with their sources but remain nearly identical.}
  \label{fig:qualitative_failuremodes}
\end{figure}

\begin{figure}[t]
  \centering
  \includegraphics[width=0.96\linewidth]{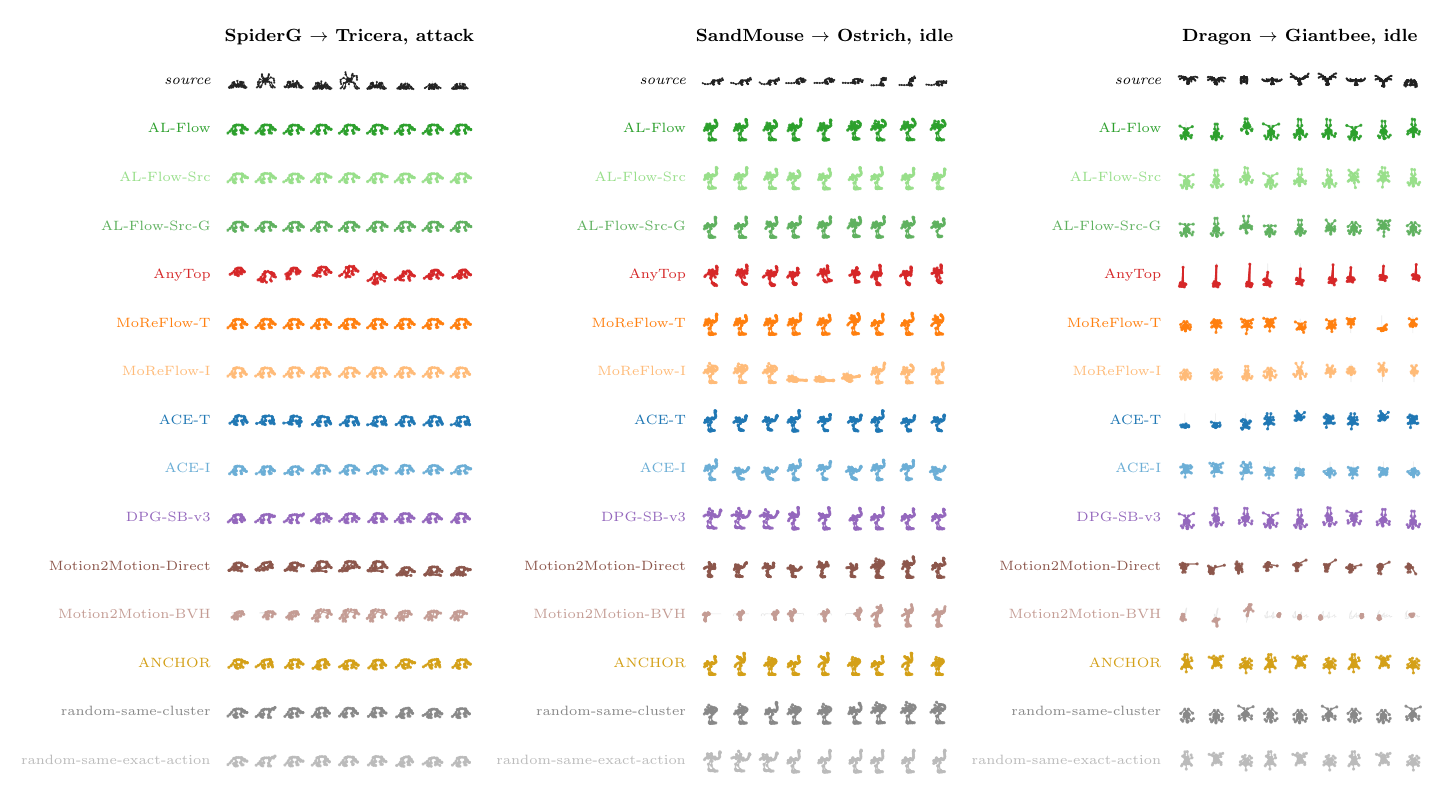}
  \caption{Full qualitative grid A. Three source clips are compared against
  target outputs from the evaluated methods on a fixed source-target-action
  triple.}
  \label{fig:qual_full_a}
\end{figure}

\begin{figure}[t]
  \centering
  \includegraphics[width=0.96\linewidth]{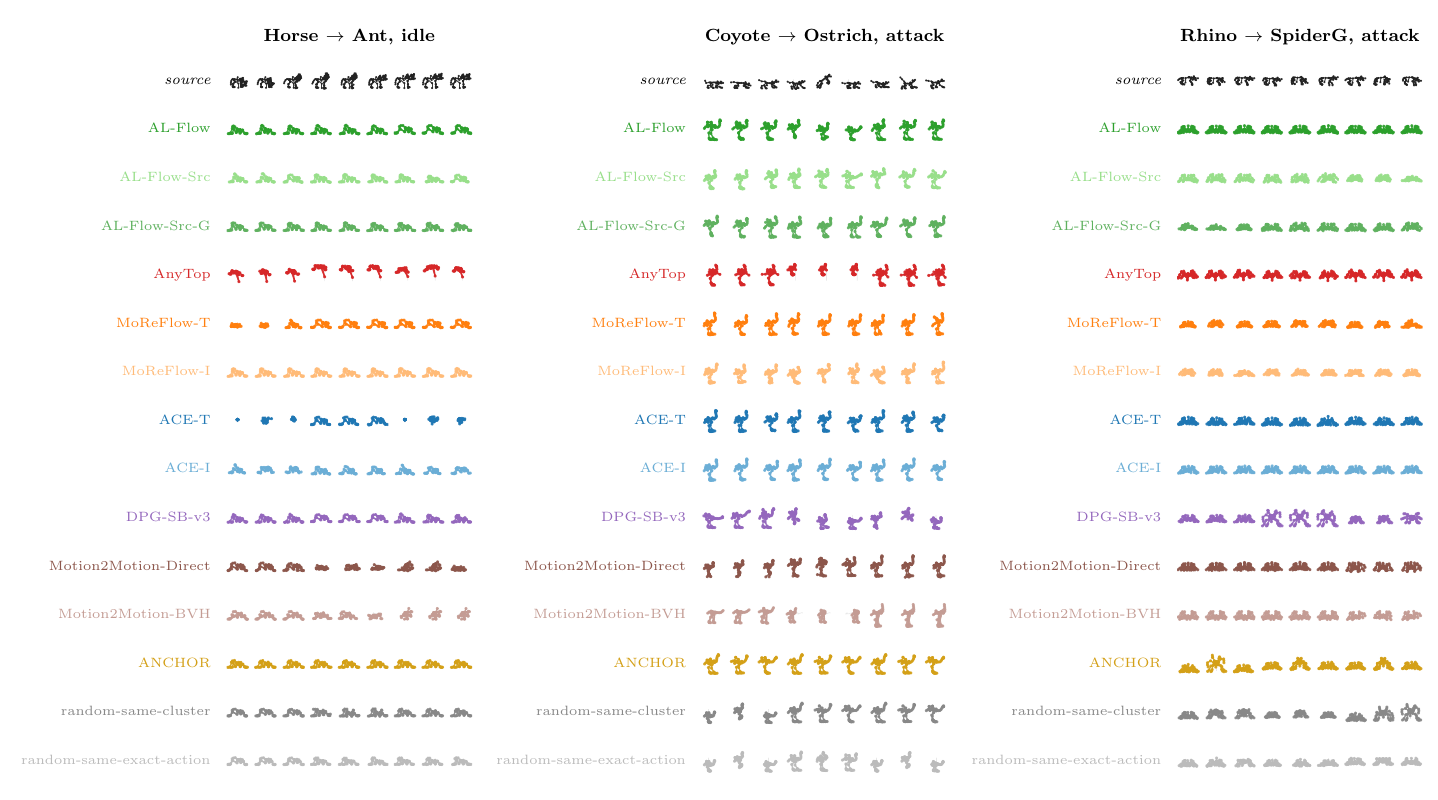}
  \caption{Full qualitative grid B. The same layout is used to reveal whether
  output variation follows the source clips after target skeleton and action are
  fixed.}
  \label{fig:qual_full_b}
\end{figure}

\begin{figure}[t]
  \centering
  \includegraphics[width=0.96\linewidth]{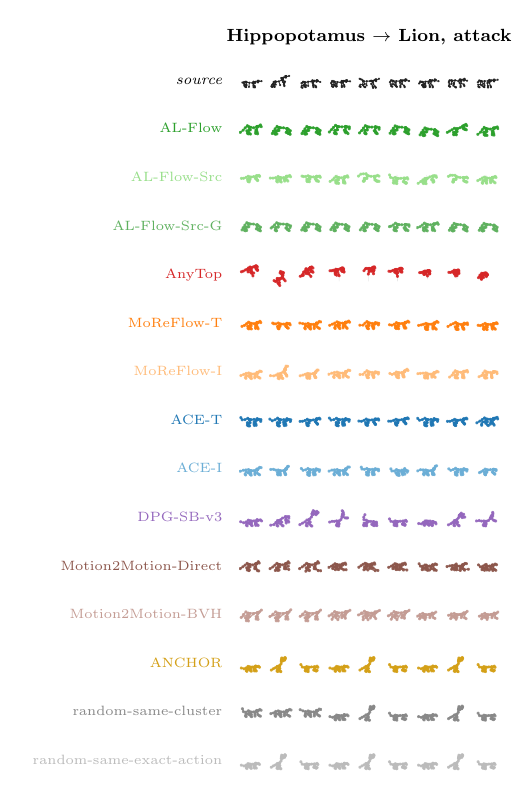}
  \caption{Full qualitative grid C. The grid complements
  Figure~\ref{fig:qualitative_failuremodes} by showing all evaluated methods on
  an additional held-out triple.}
  \label{fig:qual_full_c}
\end{figure}

The supplementary material contains nine companion videos: two main-body videos
for the panorama and failure-mode figures, and seven appendix videos for
additional source-target triples. The videos use the same projection rule as
the still figures and synchronize playback phase across methods, making
source-blind lockstep behavior and source-conditioned variation easier to see.

\section{Latent Schr\"odinger Bridge Comparator Scope}
\label{app:dpgsb_status}

The evaluated latent Schr\"odinger bridge comparator, reported under the public
method name DPG-SB-v3 in the roster, uses retrieval-initialized noise, a shared
conditioned bridge, a target-side latent constraint, and a latent-space cycle
term. The decoded-motion bone, contact, and smoothness penalties were not
included in this evaluation. This caveat is important for interpretation:
DPG-SB-v3 tests where the evaluated latent bridge variant lies relative to the
source-blind floor, not an exhaustive set of motion-space penalties for bridge
objectives.

\section{Potential Societal Impacts}
\label{app:societal_impacts}

Cross-skeleton motion generation may broaden access to motion content in
animation, education, assistive design, and scientific visualization by reducing
the need to record every action separately for every body. In these settings, a
source-preserving model could help creators transfer motion intention across
characters, support low-resource motion libraries, and make comparative studies
of movement easier. This paper contributes a diagnostic and theoretical account
of when such transfer claims are not identified by standard evidence, rather
than a deployed generator.

The same capability also carries risks. More reliable motion transfer could make
synthetic embodied media easier to produce, including misleading performances
attributed to people or characters who did not enact them, and action-level
evaluation may overstate reliability when the source-conditioned map is not
actually preserved. Biases in the source and target motion libraries may also be
propagated across bodies, especially when some anatomies, actions, or movement
styles are sparsely represented. Our analysis supports a conservative evaluation
stance: cross-skeleton methods should report source-instance diagnostics
alongside action-level scores, disclose the motion domains and licenses on which
they are trained, and avoid claims of retargeting fidelity when only
target-action recovery has been established.


\clearpage
\section*{NeurIPS Paper Checklist}

\begin{enumerate}

\item {\bf Claims}
    \item[] Question: Do the main claims made in the abstract and introduction accurately reflect the paper's contributions and scope?
    \item[] Answer: \answerYes{} 
    \item[] Justification: The abstract and introduction accurately reflect the paper's contributions and scope.

\item {\bf Limitations}
    \item[] Question: Does the paper discuss the limitations of the work performed by the authors?
    \item[] Answer: \answerYes{} 
    \item[] Justification: Section~\ref{sec:limitations} discusses the scale invariance and small-sample noise of SIF, positive scores that arise without source content, the scope of the data, and the adapted baselines.

\item {\bf Theory assumptions and proofs}
    \item[] Question: For each theoretical result, does the paper provide the full set of assumptions and a complete (and correct) proof?
    \item[] Answer: \answerYes{} 
    \item[] Justification: The paper provides the full set of assumptions and a complete (and correct) proof for each theoretical result.

    \item {\bf Experimental result reproducibility}
    \item[] Question: Does the paper fully disclose all the information needed to reproduce the main experimental results of the paper to the extent that it affects the main claims and/or conclusions of the paper (regardless of whether the code and data are provided or not)?
    \item[] Answer: \answerYes{} 
    \item[] Justification: The paper provides sufficient information to reproduce the main experimental results.

\item {\bf Open access to data and code}
    \item[] Question: Does the paper provide open access to the data and code, with sufficient instructions to faithfully reproduce the main experimental results, as described in supplemental material?
    \item[] Answer: \answerYes{} 
    \item[] Justification: The code is available. However, the Truebones dataset cannot be released due to licensing restrictions. BONES-SEED and LAFAN1 are available from their providers under their licenses.

\item {\bf Experimental setting/details}
    \item[] Question: Does the paper specify all the training and test details (e.g., data splits, hyperparameters, how they were chosen, type of optimizer) necessary to understand the results?
    \item[] Answer: \answerYes{} 
    \item[] Justification: The paper specifies all the training and test details necessary to understand the results.

\item {\bf Experiment statistical significance}
    \item[] Question: Does the paper report error bars suitably and correctly defined or other appropriate information about the statistical significance of the experiments?
    \item[] Answer: \answerYes{} 
    \item[] Justification: Section~\ref{sec:sif} and Appendix~\ref{app:sif_robustness} report 95\% bootstrap intervals that resample whole source skeletons, because source clips recur across triples, together with shuffle-test $p$-values; paired comparisons use the same resampling.

\item {\bf Experiments compute resources}
    \item[] Question: For each experiment, does the paper provide sufficient information on the computer resources (type of compute workers, memory, time of execution) needed to reproduce the experiments?
    \item[] Answer: \answerYes{} 
    \item[] Justification: The original experiments were run on an NVIDIA RTX 4070 Super GPU; the additional controls and the robot studies used one NVIDIA H200 GPU per job on a computing cluster, and the statistical analyses ran on CPUs.
    
\item {\bf Code of ethics}
    \item[] Question: Does the research conducted in the paper conform, in every respect, with the NeurIPS Code of Ethics \url{https://neurips.cc/public/EthicsGuidelines}?
    \item[] Answer: \answerYes{} 
    \item[] Justification: The research conducted in the paper conforms, in every respect, with the NeurIPS Code of Ethics.

\item {\bf Broader impacts}
    \item[] Question: Does the paper discuss both potential positive societal impacts and negative societal impacts of the work performed?
    \item[] Answer: \answerYes{} 
    \item[] Justification: The paper discusses potential societal impacts in Appendix~\ref{app:societal_impacts}.
    
\item {\bf Safeguards}
    \item[] Question: Does the paper describe safeguards that have been put in place for responsible release of data or models that have a high risk for misuse (e.g., pre-trained language models, image generators, or scraped datasets)?
    \item[] Answer: \answerNA{} 
    \item[] Justification: The paper poses no such risks.

\item {\bf Licenses for existing assets}
    \item[] Question: Are the creators or original owners of assets (e.g., code, data, models), used in the paper, properly credited and are the license and terms of use explicitly mentioned and properly respected?
    \item[] Answer: \answerYes{} 
    \item[] Justification: Every dataset and tool is credited in Section~\ref{sec:map} and Appendix~\ref{app:robots}. Truebones Zoo is used under its license and cannot be redistributed. BONES-SEED is used under the BONES-SEED License for academic research, with the attribution it requires (Motion Data by Bones Studio, \url{https://bones.studio/}). LAFAN1 is used under CC BY-NC-ND 4.0; the paper shows rendered stills for illustration and does not release the retargeted clips. GMR is used under the MIT License, together with the robot models it distributes, each credited to its original source.

\item {\bf New assets}
    \item[] Question: Are new assets introduced in the paper well documented and is the documentation provided alongside the assets?
    \item[] Answer: \answerYes{} 
    \item[] Justification: New assets introduced in the paper are well documented and the documentation is provided alongside the assets.

\item {\bf Crowdsourcing and research with human subjects}
    \item[] Question: For crowdsourcing experiments and research with human subjects, does the paper include the full text of instructions given to participants and screenshots, if applicable, as well as details about compensation (if any)? 
    \item[] Answer: \answerNA{} 
    \item[] Justification: The paper does not involve crowdsourcing nor research with human subjects.

\item {\bf Institutional review board (IRB) approvals or equivalent for research with human subjects}
    \item[] Question: Does the paper describe potential risks incurred by study participants, whether such risks were disclosed to the subjects, and whether Institutional Review Board (IRB) approvals (or an equivalent approval/review based on the requirements of your country or institution) were obtained?
    \item[] Answer: \answerNA{} 
    \item[] Justification: The paper does not involve crowdsourcing nor research with human subjects.

\item {\bf Declaration of LLM usage}
    \item[] Question: Does the paper describe the usage of LLMs if it is an important, original, or non-standard component of the core methods in this research? Note that if the LLM is used only for writing, editing, or formatting purposes and does \emph{not} impact the core methodology, scientific rigor, or originality of the research, declaration is not required.
    \item[] Answer: \answerNA{} 
    \item[] Justification: The LLM is used only for writing, editing, or formatting purposes.

\end{enumerate}

\end{document}